\documentclass[letterpaper,11pt]{article}
\usepackage[utf8]{inputenc}

\usepackage{fullpage}  
\usepackage{amsthm}
\usepackage{amsmath,amssymb, mathtools}
\usepackage{newtxmath}
\usepackage{framed,caption,color,url}
\usepackage{times}
\usepackage{algorithm2e,algorithmicx,algpseudocode,tikz,wrapfig}
\usepackage{graphicx}               
\usepackage{subfigure}
\usepackage{multirow}
\usepackage{nicefrac} 
\usepackage{paralist} 
\usepackage{boxedminipage} 
\usepackage{upgreek}
\usepackage{multicol}
\usepackage{enumitem,comment}
\usepackage[pdfstartview=FitH,colorlinks,linkcolor=blue,filecolor=blue,citecolor=blue,urlcolor=blue,pagebackref=true]{hyperref}
\usepackage[capitalize]{cleveref} 

\newcommand{\DInfWith}{\mathsf{Del}\text{-}\mathsf{Inf}\text{-}\mathsf{Exm}}
\newcommand{\DInfNo}{\mathsf{Del}\text{-}\mathsf{Inf}\text{-}\mathsf{Ins}}
\newcommand{\DelRecon}{\mathsf{Del}\text{-}\mathsf{Ins}\text{-}\mathsf{Rec}}

\newcommand{\LabelApp}{\mathsf{Ins}\text{-}\mathsf{Rev}\text{-}\mathsf{Lbl}\text{-}\mathsf{Rec}}
\newcommand{\LabelInf}{\mathsf{Del}\text{-}\mathsf{Lbl}\text{-}\mathsf{Rec}}
\usepackage{centernot}

\newcommand{\Learn}{\mathsf{Learn}}
\newcommand{\ERM}{\mathsf{ERM}}

\newcommand{\Add}{\mathsf{Add}}
\newcommand{\Del}{\mathsf{Del}}
\newcommand{\Eval}{\mathsf{Eval}}

\newcommand{\DataCol}{\mathsf{DatCol}}
\newcommand{\DelReq}{\mathsf{DelReq}}
\newcommand{\Env}{\mathsf{Env}}
\newcommand{\Records}{\mathcal{U}}

\newcommand{\real}{\mathsf{real}}
\newcommand{\ideal}{\mathsf{ideal}}

\newcommand{\Risk}{\mathsf{Risk}}

\newcommand{\aSF}{\mathsf{a}}
\newcommand{\gSF}{\mathsf{g}}
\newcommand{\hSF}{\mathsf{h}}

\newcommand{\avr}[2]{\ifthenelse{\equal{#2}{}}{\aSF({#1})}{\ifthenelse{\equal{#2}{0}}{\aSF(\emptyset)}{\aSF({#1}_{\leq #2})}}}

\newcommand{\avrMax}[2]{\ifthenelse{\equal{#2}{}}{\aSF^*({#1})}{\ifthenelse{\equal{#2}{0}}{\aSF^*(\emptyset)}{\aSF^*({#1}_{\leq #2})}}}

\newcommand{\avrApp}[2]{\ifthenelse{\equal{#2}{}}{\tilde{\aSF}({#1})}{\ifthenelse{\equal{#2}{0}}{\tilde{\aSF}(\emptyset)}{\tilde{\aSF}({#1}_{\leq #2})}}}

\newcommand{\avrAppMax}[2]{\ifthenelse{\equal{#2}{}}{\tilde{\aSF}^*({#1})}{\ifthenelse{\equal{#2}{0}}{\tilde{\aSF}^*(\emptyset)}{\tilde{\aSF}^*({#1}_{\leq #2})}}}

\newcommand{\ArgMax}[2]{\ifthenelse{\equal{#2}{}}{\hSF({#1})}{\ifthenelse{\equal{#2}{0}}{\hSF(\emptyset)}{\hSF({#1}_{\leq #2})}}}

\newcommand{\AppArgMax}[2]{\ifthenelse{\equal{#2}{}}{\tilde{\hSF}({#1})}{\ifthenelse{\equal{#2}{0}}{\tilde{\hSF}(\emptyset)}{\tilde{\hSF}({#1}_{\leq #2})}}}

\newcommand{\gain}[2]{\ifthenelse{\equal{#2}{}}{\gSF(#1)}{\gSF(#1_{\leq #2})}}
\newcommand{\gainMax}[2]{\ifthenelse{\equal{#2}{}}{\gSF^*(#1)}{\gSF^*(#1_{\leq #2})}}
\newcommand{\gainApp}[2]{\ifthenelse{\equal{#2}{}}{\tilde{\gSF}(#1)}{\tilde{\gSF}(#1_{\leq #2})}}
\newcommand{\gainAppMax}[2]{\ifthenelse{\equal{#2}{}}{\tilde{\gSF}^*(#1)}{\tilde{\gSF}^*(#1_{\leq #2})}}

\newcommand{\dis}{\mathsf{dis}}

\newcommand{\sm}{\setminus}

\newcommand{\Chal}{\mathsf{Chal}}

\newcommand{\model}{h}
\newcommand{\modelD}[1]{\model_{-#1}}
\newcommand{\modelDel}{\model_{\mathsf{del}}}

\newcommand{\loss}{\ell}

\usepackage{graphicx}

\newcommand{\remove}[1]{}

\newcommand{\se}{\subseteq}

\newcommand{\set}[1]{\{ #1 \}}

\newcommand{\bits}{\{0,1\}}

\newcommand{\R}{{\mathbb R}}

\newcommand{\Adv}{\mathsf{Adv}}

\newcommand{\cC}{{\mathcal C}}
\newcommand{\cD}{{\mathcal D}}

\newcommand{\cH}{{\mathcal H}}

\newcommand{\cS}{{\mathcal S}}
\newcommand{\cT}{{\mathcal T}}

\newcommand{\cX}{{\mathcal X}}
\newcommand{\cY}{{\mathcal Y}}

\newcommand{\bfe}{\mathbf{e}}

\newcommand{\bfx}{\mathbf{x}}

\usepackage{upgreek}
\usepackage{bm}

\newcommand{\eps}{\varepsilon}

\newcommand{\Exp}{\operatorname*{\mathbb{E}}}
\newcommand{\Ex}{\Exp}

\newcommand{\Supp}{\operatorname{Supp}}

\newcommand{\argmax}{\operatorname*{argmax}}
\newcommand{\argmin}{\operatorname*{argmin}}

\newcommand{\one}{\vmathbb{1}}

\newtheorem{theorem}{Theorem}[section]

\theoremstyle{plain}

\newtheorem{lemma}[theorem]{Lemma}

\theoremstyle{definition}

\newtheorem{definition/}[theorem]{Definition}
\newenvironment{definition}
  {%
   \pushQED{\qed}\begin{definition/}}
  {\popQED\end{definition/}}

\newtheorem{alg/}[theorem]{Algorithm}

\newenvironment{alg}
  {%
   \pushQED{\qed}\begin{alg/}}
  {\popQED\end{alg/}}

\theoremstyle{definition}
\newtheorem{remark}[theorem]{Remark}

\newcommand{\sdotfill}{\textcolor[rgb]{0.8,0.8,0.8}{\dotfill}} 

\makeatletter
\def\th@protocol{%
    \normalfont 
    \setbeamercolor{block title example}{bg=orange,fg=white}
    \setbeamercolor{block body example}{bg=orange!20,fg=black}
    \def\inserttheoremblockenv{exampleblock}
  }
\makeatother
\theoremstyle{protocol}

\newtheorem{proto}[theorem]{Protocol}
\newtheorem{protoc}[theorem]{Protocol}

\newcommand{\namedref}[2]{#1~\ref{#2}}

\newcommand{\torestate}[3]{%
\expandafter \def \csname BBRESTATE #2 \endcsname{#3}
\theoremstyle{plain}
\newtheorem{BBRESTATETHMNUM#2}[theorem]{#1}
\begin{BBRESTATETHMNUM#2}\label{#2}\csname BBRESTATE #2 \endcsname   \end{BBRESTATETHMNUM#2}
\newtheorem*{BBRESTATETHMNONNUM#2}{\namedref{#1}{#2}}
}

\newcommand{\restate}[1]{\begin{BBRESTATETHMNONNUM#1}[Restated] \csname BBRESTATE #1 \endcsname
\end{BBRESTATETHMNONNUM#1}}

\newcommand{\state}         {\mathnormal{state}} 
\newcommand{\view}          {\mathnormal{view}}

\title{Deletion Inference, Reconstruction, and Compliance  in \\ Machine (Un)Learning \footnote{This is the full version of a paper appearing in the 22nd Privacy Enhancing Technologies Symposium (PETS 2022).}}
\author{Ji Gao\thanks{University of Virginia, \href{mailto:jg6yd@virginia.edu}{jg6yd@virginia.edu}.} \and Sanjam Garg\thanks{University of California, Berkeley and NTT Research,  \href{mailto:sanjamg@berkeley.edu}{sanjamg@berkeley.edu}.  Supported in part by DARPA under Agreement No. HR00112020026, AFOSR Award FA9550-19-1-0200, NSF CNS Award 1936826, and research grants by the Sloan Foundation, and Visa Inc. Any opinions, findings and conclusions or recommendations expressed in this material are those of the author(s) and do not necessarily reflect the views of the United States Government or DARPA.} \and Mohammad Mahmoody\thanks{University of Virginia, \href{mailto:mohammad@virginia.edu}{mohammad@virginia.edu}. Supported by NSF grants CCF-1910681 and CNS1936799.} \and Prashant Nalini Vasudevan\thanks{National University of Singapore, \href{mailto:prashant@comp.nus.edu.sg}{prashant@comp.nus.edu.sg}. Supported by funds from an NUS Presidential Young Professorship. Part of this work was done as a postdoctoral researcher at UC Berkeley supported by Sanjam Garg's funds listed above and the UC Berkeley Center for Long-Term Cybersecurity.}
}

\newcommand{\ifcomments}{\iftrue}

\ifcomments
\newcommand{\Mnote}[1]{{\footnotesize \color{teal} [\bf {Mohammad:}  #1]}}
\newcommand{\sanjam}[1]{{\footnotesize \color{red} [\bf {Sanjam:}  #1]}}
\newcommand{\Jnote}[1]{{\footnotesize \color{red} [\bf {Ji:}  #1]}}
\newcommand{\Pnote}[1]{{\footnotesize {\color{red} [{Prashant:}  #1]}}}

\else
\newcommand{\Mnote}[1]{}
\newcommand{\sanjam}[1]{}
\newcommand{\Jnote}[1]{}
\newcommand{\Pnote}[1]{}
\fi

\begin{document}

\maketitle

  \begin{abstract}
{
Privacy attacks on machine learning models aim to identify the data that is used to train such models. Such attacks, traditionally, are studied on \emph{static} models that are trained once and are accessible by the adversary.
Motivated to meet new legal requirements, many machine learning methods are recently extended to support machine \emph{unlearning}, i.e., updating models as if certain examples are removed from their training sets, and meet new legal requirements. However, privacy   attacks could  potentially become  more devastating in this new setting, since an attacker could now access both the original model before deletion and  the new model after the deletion. In fact, the very act of deletion might make the deleted record \emph{more} vulnerable to privacy attacks.
%

Inspired by cryptographic definitions and the differential privacy framework, we  \emph{formally} study privacy implications of machine unlearning.
We formalize (various forms of)  \emph{deletion inference} and \emph{deletion reconstruction} attacks, in which the adversary aims to either identify which record is deleted or to reconstruct (perhaps part of) the    deleted records.
We then  present successful {deletion inference} and  reconstruction attacks for a variety of machine learning models and tasks such as classification, regression, and language models.
%
%
Finally, we show that our attacks would provably be precluded if the schemes satisfy (variants of)    
Deletion Compliance (Garg, Goldwasser, and Vasudevan, Eurocrypt'20).
}
\end{abstract}

\remove{
\newcommand{\chaptermark}[1]{\markright{\textsf{CH.}\hspace{1em} {\small}.#1}{}}
\makeatletter
\newcommand{\@makeschapterhead}[1]{}
\makeatother
}
\newpage
\tableofcontents
 




\section{Introduction} \label{sec:intro}
Machine learning algorithms, in their most basic settings, focus on deriving predictive models with low error by using a collection of training examples $\cS=\set{e_1,\dots,e_n}$.
However, a model $\model_\cS$ trained on set $\cS$ might reveal (sensitive information about) the examples in $\cS$, potentially violating the privacy of the individuals whose contributed those examples. Such exposure, particularly in certain  (e.g., medical/political) contexts could be a major concern. In fact, the ever-increasing use of machine learning (ML) as a service~\cite{ribeiro2015mlaas}  for decision making  further heightens such privacy concerns. Recent legal requirements (e.g., the European Union's GDPR~\cite{hoofnagle2019european} or California's CCPA~\cite{de2018guide}) aim to make such privacy considerations mandatory. At the same time, a recent line of work~\cite{VBE18,CN19,nissim2017bridging,GGV20} aims at (mathematically) formalizing such privacy considerations and their enforcement.


The work of Shokri et al.~\cite{shokri2017membership} demonstrated that natural and even commercialized ML models do, in fact, leak a lot about their training sets. In particular, their work initiated the \emph{membership inference} framework for studying privacy attacks on ML models. In such  attacks, an adversary with input example $e$ and access to an ML model $\model_\cS$ aims to deduce if  $e \in \cS$ or not. In a bigger picture, membership inference of \cite{shokri2017membership} and many follow-up attacks~\cite{long2017towards,salem2019ml,long2018understanding,yang2019neural,choo2020label,li2020label,zanella-bguelin2020analyzing,salem2020updates,jayaraman2020revisiting} as well as \emph{model inversion} attacks \cite{fredrikson2014privacy,fredrikson2015model,wu2016methodology,VBE18} can all be seen as demonstrating ways to \emph{infer} or \emph{reconstruct} information about the data sets used in the ML pipeline
based on publicly available auxiliary information about them 
\cite{dinur2003revealing,dwork2017exposed,backes2016membership,dwork2015robust,sankararaman2009genomic,homer2008resolving}. A more recent line of work studies the related question of ``memorization'' in machine learning models  set~\cite{song2017machine,VBE18,carlini2019secret,feldman2020does}.

On the defense side, differential privacy~\cite{dinur2003revealing,TCC:DMNS06,dwork2008differential}  provides a framework to provably limit the information that would leak about the used training examples. This is done by guaranteeing that including or not including any individual example will have little statistical impact on the distribution of the produced ML model.  Consequently, any form of interaction with the trained model $\model$ (e.g., even a full  disclosure of it) will  not reveal too much information about whether a particular example $e$ was a member of the data set or not. Despite being a very powerful privacy guarantee, differential privacy imposes a  challenge on the learning process~\cite{song2013stochastic,dwork2014analyze,bassily2014private,dandekar2018differential,sheffet2019old,talwar2015nearly} that usually leads to major utility loss when one uses the same amount of training data compared with non-private training 
\cite{bassily2014private,beimel2014bounds}. 
Hence, it is important to  understand the level of privacy  that can be achieved by more efficient  methods as well.

\paragraph{Privacy  in the presence of data deletion.} The above mentioned attacks  are executed in a \emph{static}   setting, in which the model is trained once and then the  adversary tries to extract information about the training set by interacting with the trained model afterwards. However, this setting is not   realistic when models are dynamic and get updated.
%
In particular, in light of the recent attention to the ``right to erasure'' or the ``right to be forgotten,''  also stressed by legal requirements such as GDPR and CCPA, a new line of work has emerged with the goal of \emph{unlearning} or simply \emph{deleting} examples from  machine learning models~\cite{cao2015towards,ginart2019making,golatkar2019eternal,GGV20,bourtoule2019machine,izzo2020approximate,guo2019certified,neel2020descent}. In this setting, upon a deletion request for an example $e \in \cS$, the trainer needs to update the model $\model_\cS$ to $\modelD{e}$ such that $\modelD{e}$  (ideally) has the same distribution as training a model from  scratch using $\cS \sm \set{e}$. 
Clearly, if an ML model gets updated due to a deletion request, we are no longer dealing with a static  ML model.


It might initially seem like a perfect deletion of a example $e$  from a model $\model_\cS$ and releasing $\modelD{e}$ instead  should  \emph{help} with preventing   leakage about the particular deleted  example $e$. After all, we are  \emph{removing} $e$ from the learning process of the model accessible to the adversary. However, the adversary now could potentially access  \emph{both} models $\model_\cS$ and $\modelD{e}$, and so it might be able to extract even \emph{more} information about the deleted example $e$ compared to the setting in which the adversary could only access $\model_\cS$ or $\modelD{e}$ alone. As a simplified   contrived example, suppose the examples $\cS=\set{e_1,\dots,e_n}$ are real-valued vectors, and suppose   the ML model $\model_\cS$ (perhaps upon many queries) somehow reveals the summation $\sum_{i \in [n]} e_i$. In this case, if the set $\cS$ is sampled from a distribution  with sufficient entropy, the trained model $\model_\cS$ might potentially provide a certain degree of privacy for examples in $\cS$. However, if one of the examples $e_i$ is deleted from $\model_\cS$, then because the updated model $\modelD{e_i}$ also returns the updated summation $\sum_{j \neq i} e_j$, then an adversary who extracts both of these summations can reconstruct the deleted record $e_i$ completely. In other words, the very task of deletion might in fact \emph{harm} the privacy of the very deleted example $e_i$.  Hence, in this work we ask:
 How vulnerable are   ML algorithms to leak information about the  deleted examples, if an adversary gets to interact with the models both before and after the deletion updates?


\subsection{Our Contribution} 
In this work, we formally study the privacy implications of machine unlearning. Our approach is inspired by cryptographic definitions, differential privacy, and deletion compliance framework of \cite{GGV20}.
More specifically, our contribution is two-fold. First, we initiate a formal study of various attack models in the two categories  of  \emph{reconstruction} and \emph{inference} attacks.   Second, we present practical, simple, yet effective attacks on a \emph{broad} class of machine learning algorithms for classification, regression, and text generation that extract information about the deleted example. 

Below, we briefly go over new definitions, the relation between them, and the ideas behind our attacks. In what follows, $\model_\cS$ is the model trained on the set $\cS$, and $\modelD{e}$ is the model after deletion of the example $e \in \cS$. When the context is clear, we might simply use $\model$ to denote $\model_\cS$ and $\modelDel$ to denote the model after deletion\footnote{Using $\modelDel$ is particularly useful when we want to refer to the model after deletion, without explicitly revealing the deleted example $e$.}. We assume that the deletion is \emph{ideal}, in the sense that $\modelD{e}$ is obtained by a fresh retrain on $\cS \sm \set{e}$.\footnote{We suspect our attacks should have a good success rate on ``approximate'' deletion procedures (in which $\modelD{e}$ is just close to the ideal version) as well. We leave such studies for future work.} The adversary will have access to $\model_\cS$ followed by access to $\modelD{e}$.

%
%
%
%

\paragraph{Deletion inference.} Perhaps the most natural question about data leakage in the context of machine unlearning is whether deletion can be inferred. In \emph{membership} inference attack, the job of the adversary is to infer whether an example $e$ is a member of the used training set $\cS $ or not by interacting with the produced model $\model_\cS$. In this work, we introduce \emph{deletion inference} attacks which are, roughly speaking, analogous to \emph{membership inference} but in the context where some deletion is happening. More specifically, our definition does not capture whether the deletion is happening or not, and our goal (in the main default definition) is only to hide \emph{which} examples are being deleted. In particular, we formalize the goal of a deletion inference adversary to distinguish between a data example $e \in \cS$ that was deleted from an ML model $\model_\cS$ and another example $e' \in \cS$ (or $e' \notin \cS)$ that is not deleted from $\cS$. We follow the cryptographic game-based style of security definitions. (See Definition~\ref{def:WeakDelPriv} for the formal definition.)
    
Given examples $e_0,e_1$, with the promise that one of them is deleted and the other is not,  one can always reduce the goal of a deletion inference adversary to   \emph{membership inference} by first inferring membership of $e_0,e_1$ in the two models $\model,\modelDel$.  However, given that the adversary has access to both of $\model,\modelDel$, it is reasonable to suspect that much more can be done by a deletion inference adversary than what can be done through a reduction to membership inference. In fact, this is exactly what we show in Section~\ref{sec:compare}.
We show that when both models $\model,\modelDel$ can be accessed, relatively simple attacks can be designed to distinguish the deleted examples from the other examples  by relying on the intuition that a useful model is usually more fit to the training data than to other data. 
In Section~\ref{sec:del-infer}, we show the power of such attacks on a variety of models and real world data sets   for both regression and classification. In each case, we both study deletion inference adversaries who know the full labeled examples $e_0,e_1$ (and infer which one of them are deleted) as well as stronger attackers who only know the (unlabeled) \emph{instances} $x_0,x_1$. 

\paragraph{Deletion reconstruction.}
The second category of our attacks focus on \emph{reconstructing} part or all of the deleted example $e$. As anticipated, reconstruction attacks are \emph{stronger} (and hence harder to achieve) attacks that can be used for obtaining deletion inference attacks as well (see Theorem~\ref{thm:fromRecToInf}). In all of our reconstruction attacks, the adversary is not given any explicit examples, and its goal is to extract information about   the features of the deleted instance. We now describe some special cases of reconstruction attacks that we particularly study.
\begin{itemize}
\item \textbf{Deleted instance reconstruction.} Can an adversary fully or  approximate find  the features of a deleted instance $x$ (where $e=(x,y)$ is the deleted example)? We show that for natural data distributions (both theoretical   and real data) the 1-nearest neighbor classifier can completely reveal the deleted instance, even if the adversary has only black-box access to the models before and after deletion. In particular, we show that when the instances are uniformly distributed over $\bits^d$,
and the model is the 1-nearest neighbor model, an adversary can extract virtually \emph{all} of the features of the deleted instance (see Section~\ref{sec:Singleton}). 
%
We also present attacks on real data for two major application settings: image classification and text generation.

\begin{itemize}
\item \textbf{Deleted image   reconstruction.} We show similar attacks on 1-nearest neighbor over the   Omniglot dataset, where the job of the adversary is to extract \emph{visually similar} pictures to the deleted ones (see Section~\ref{sec:visual}).
\item \textbf{Deleted sentence reconstruction.} We then  study deletion reconstruction attacks on language models. Here, a language model gets updated to remove an input (e.g., a sentence) $e$, and the job of the adversary is to find useful information about $e$. We show that for simple language models such as bigram or trigram models, the adversary can  extract $e$ completely. 
\end{itemize}

\item \textbf{Deleted label reconstruction.} Suppose we deal with a classification problem. For a deleted example $(x,y)=e$, can an adversary \emph{who does not know the instance  $x$} infer any information about the label of the deleted point? We show that this is indeed possible with a simple idea when the data set is not too large. In particular, the deletion of a point with label $c$ reduces the probability that the new model outputs label $c$ in general, and using this idea we give simple yet successful attacks.
Now, suppose the adversary is somehow \emph{aware} of the instance $x$  of a deleted example $(x,y)=e$. Can the adversary leverage knowing the instance $x$ to learn \emph{more information} about the label $y$, than each of the models $\model,\modelDel$ alone provide?  %
We show that doing so is possible for linear regression. In particular, we show an attack using which one can extrapolate a deleted point's label to a \emph{higher} precision than what is  provided through the original model $\model$ or the model after deletion $\modelDel$. (See Section~\ref{sec:approx}.)

\end{itemize}

\paragraph{Weak deletion compliance.} The above results all deal with first defining \emph{attack models} and then presenting attacks within those frameworks. Next, we ask if it is possible to realize  machine learning algorithms with deletion mechanisms that offer meaningful notions of privacy for the deleted points. We approach this question through the lens of the  recent work of Garg et al.~\cite{GGV20} in which they provide a general ``deletion complience'' framework that provides strong definitions of private data deletion. We first give a formal comparison between the framework of~\cite{GGV20} with our  attack models and show that the deletion compliance framework of~\cite{GGV20} indeed captures all of the above-mentioned attack models. Furthermore, we also present a \emph{weakened} variant of the definition of~\cite{GGV20} that is adapted to a setting where the fact that deletion happened itself is allowed to leak. We believe this is a natural setting that needs special attention. For example, consider a text with redacted parts; this reveals the fact that deletion has happened, but not necessarily the redacted text. We further weaken the framework of~\cite{GGV20} by only revealing to the adversary what can be accessed through \emph{black-box access} to the model and \emph{not} the full state of the model.  We show that even such weaker variants of deletion compliance still capture all of our attacks, and hence is sufficient for positive results. This means that, as shown by \cite{GGV20}, differential privacy (with strong parameters) can be used to prevent all attacks of our paper. However, note that enforcing differential privacy comes with costs in efficiency and sample complexity. Hence, it remains an interesting direction to find more efficient schemes (both in terms of running time and sample complexity) that satisfy our weaker notions of deletion compliance introduced in this work.  See Section~\ref{sec:GGV} for more discussions.

\paragraph{Motivation behind the attacks.} At a high level, our work is relevant in any context in which (1) the users who provide the data examples care about their privacy and prefer not to reveal their participation in the data set $\cS$ (2) the system aims to provide the deletion operation, perhaps due to legal requirements. Condition (1)   essentially holds in any scenario in which membership inference constitutes a legitimate threat. In scenarios where conditions (1) and (2) hold, if the adversary maintains continuous access to the machine learning model (e.g., when the model is provided as public service) then all the attacks studied in this paper are relevant to practice and would model  different adversarial power. 

Our security games model attacks in which the adversary aims to infer (or reconstruct) deletion of a \emph{random} example from a dataset. Real world adversaries are stronger in the sense that they could have a specific target in mind before making their queries to the online model. Moreover, real world adversaries usually have a lot of auxiliary information (e.g., as those exploited in the attacks on privacy on users in the Netflix challenge \cite{narayanan2006break}) while our attackers have a minimal knowledge about the distribution from which the data is sampled.

Having a diverse set of security games and attacks is analogous to having many different security games and notions in cryptography (such as CPA and CCA security for encryption) to model different attack scenarios. Informally speaking, and at a very high level, one can also think of the very strong deletion compliance of \cite{GGV20} as ``UC security'' \cite{FOCS:Canetti01}, while our other security notions model weaker security criteria.

\subsection{Related Work}


Chen et al~\cite{chen2020machine} study a setting similar to ours. They show attacks that, given access to two models -- one trained on a dataset $\cS$ and another on $\cS\setminus\set{e}$ -- determine whether a given input $e'$ is equal to the deleted item $e$. This is close to our notion of deletion inference, though not quite the same. They show that their attacks perform much better than plain membership inference on the first model. Our work differs from that of \cite{chen2020machine} in the following  respects:
\begin{enumerate}
    \item In addition to deletion inference, we also show various kinds of reconstruction attacks in a variety of models with different reconstruction goals.
    \item Their attacks are constructed by running sophisticated learning algorithms on the posteriors corresponding to deleted and not deleted samples. While this results in attacks that work quite well, these attacks have little explanatory power -- it is not clear what enables them, and it is hard to tell what the best way to prevent them is. Our attacks, on the other hand, make use of simple statistics of the outputs of the models.
    \item They show that certain measures like publishing only the predicted label or using differential privacy can stop their attacks from working, but this is far from showing that such measures prevent all possible attacks. In order to prove security against all attacks, a formalization of what entails such security is necessary. We provide formal definitions of privacy and formally build a connection to the deletion compliance framework of \cite{GGV20}, which, as corollary, implies that differential privacy can provably prevent any possible deletion inference attack.
\end{enumerate}
 
The work of Salem et al~\cite{salem2020updates} also studies a related setting. In their case, a model is updated by the addition of new samples, rather than by deletion, and they show attacks that partially reconstruct either the new sample itself or its label. These attacks are constructed by training generative models on posteriors of various samples from a shadow model. It is  possible that their attacks can be used when data is {deleted} as well. In fact,   our attacks can also potentially be adapted to be applied when the data is \emph{added} rather than deleted (but the security game needs to change to formally allow this). They also present a cursory discussion of possible defences against their attacks, suggesting that adding noise to the posteriors or differential privacy might work. The distinction of our work from theirs is along the same lines as above -- our attacks are simpler and more transparent, and our formalization allows us to identify strategies for \emph{provable} security against arbitrary attacks by proving the relation of our attacks and the deletion compliance of \cite{GGV20}. On the attack side, our work studies the attack landscape with much more granularity by studying very specific attacks that aim to only reconstruct (or infer) the instances, or their labels, or leverage the knowledge of the instance to better approximate the labels.


\section{Preliminaries}
\paragraph{Basic Notation.} $[n] $ denotes $\set{1,\dots,n}$. $\cX$ denotes the \emph{instance} space, and $\cY$ denotes the \emph{label}  space. For regression tasks, $\cY=\R$ is the set of real numbers, and for classification tasks $\cY$ is a finite set where by default $\cY=\bits$.
$D$ denotes a distribution over $\cX \times \cY$, and $D^n$ denotes the $n$-fold product of $D$. A sample   $e=(x,y) \gets D$ is called a (labeled) \emph{example}. By $D\equiv D'$ we denote that   $D,D'$ are identically distributed.   When the data examples are not necessarily iid sampled, we use $S_n$ to denote a \emph{distribution} over data sets of size $n$ (one special case is $S_n \equiv D^n$) and we use $\cS \gets S_n$ to denote   sampling  $\cS$ from $S_n$. $\cH \se \cY^\cX$  denotes a set of \emph{models} (aka \emph{hypothesis class})  mapping $\cX$ to $\cY$. For example, $\cH$ could be the set of all neural nets with a specific architecture and size or the set of half spaces in dimension $d$ when $\cX = \R^d$.

\paragraph{Loss, risk, and  learning.}  A loss function $\loss \colon \cH \times \cX \times \cY \to \R$ maps an input $(h,(x,y))$ to $\R$ and  measures  how bad the prediction of $h$ on $x$ is compared to the true label $y$. For classification,  we use the 0-1 loss   $\loss(h,e)=\one[h(x)\neq y]$, where $\one$ is the Boolean indicator random variable. $\Learn$ denotes a (perhaps randomized) learner that maps any (unordered) set of examples $\cS=\set{e_1,\dots,e_n}$ to a model $h \in \cH$.  $\Risk_D(h)=\Ex_{e \gets D} \loss (h,e)$ denotes the  \emph{population risk} of $h$ over a distribution $D$.  $\Risk_\cS(h)=\Ex_{e \gets \cS} \loss (h,e)$ denotes the \emph{empirical risk}  of $h$ over a training set $\cS \subset (\cX \times \cY)^*$.  The Empirical Risk Minimization   rule  $\ERM$ is the learner that simply outputs a model that minimizes the empirical loss
$\ERM(\cS) = \argmin_{h \in \cH} \Risk_\cS(h)$. 
 


\paragraph{Deletion.} Fix a learner $\Learn$, training set $\cS$, and model $h \gets \Learn(\cS)$. We use $h_{-e} \gets \Del_\cS(h,e)$ to denote the ``ideal'' data deletion procedure \cite{ginart2019making} that outputs $h_{-e} \gets \Learn(\cS \sm \set{e})$ using fresh randomness for $\Learn$ if needed. (Hence, if $e \not \in \cS$, then $\Del_\cS(h,e)$ simply returns a fresh retraining on $\cS$.) In general, $\Del$ needs to know the training set on which $h$ is trained, or it needs a data structure that keeps some information about $\cS$ in addition to $h$. Whenever $\cS$ is clear from the context,  we might simply write $h_{-e} \gets \Del(h,e)$. 

\section{Deletion Inference Attacks}
\label{sec:del-infer}

In this section, we describe   a framework of attacks  on machine unlearning (i.e., machine learning with deletion option) schemes  that  can \emph{infer} the deleted examples. Such attacks are executed by adversaries who first access the model before deletion followed by having access to  the model after deletion. In each case, we will first formally explain our threat model. We also  provide  theoretical intuition behind our attacks and     report   experimental findings by implementing those attacks. 
%




\subsection{Threat Model}
We define a security game that captures how well an adversary can tell which element is being deleted from the training set. Note that our (default) definition is not aiming to hide the fact that \emph{something} is being deleted, and the only thing we try to hide is \emph{which} element is being deleted. We use a definition that is inspired by how (CPA or CCA) security of encryption schemes are defined through indistinguishability-based security games \cite{goldwasser1984probabilistic,naor1990public}.
 

\begin{definition}[Deletion inference] \label{def:WeakDelPriv}
Let $\Learn$ be a learner, $\Del$ be a deletion mechanism for $\Learn$, and $S_n$ be a distribution on datasets of size $n$. The adversary $\Adv$ and the challenger $\Chal$ interact as follows.
\begin{enumerate}
    \item {\bf Sampling the data and revealing the challenges.} $\Chal$ picks a dataset $\set{z_1,\dots,z_n}=\cS \gets S_n$ of size $n$. $\Chal$   picks two indices $i\neq j \in [n]$ at random and sends $e_0=z_i, e_1=z_j$  to $\Adv$.
        \item {\bf Oracle access  before deletion.}  $\Chal$     trains $\model \gets \Learn(\cS)$.
   $\Adv$ is then given \emph{oracle} access to $\model$, and finally instructs moving to the next step.
  
      \item {\bf Random selection and deletion.} $\Chal$ picks   $b \gets \bits$ at random and lets $\modelDel \gets \Del(h,e_b)$.
     \item {\bf Oracle access after deletion.}     
        The adversary $\Adv$   is now given {oracle} access (only) to  $\modelDel$.
    
    \item {\bf Adversary's guess.}  The adversary  sends out a bit  $b'$ to $\Chal$ and wins if $b'=b$.
\end{enumerate}

The scheme $(\Learn,\Del)$ is called \emph{$\rho$  insecure  against deletion inference} for data distribution $S_n$, if there is a PPT adversary $\Adv$   whose success probability in the game above is at least $\rho$.  (Note that achieving $\rho=1/2$ is trivial.) 
Now, consider a modified  game in which the adversary is  given \emph{only the instances} $(x_0,x_1)$ where $e_0=(x_0,y_0),e_1=(x_1,y_1)$. We call this game the \emph{instance} deletion inference. If an adversary has success probability   at least   $\rho$  in the instance deletion inference game, then the scheme $(\Learn,\Del)$ is called $\rho$  insecure  against \emph{instance} deletion inference for distribution $S_n$. Similarly, we define \emph{label} deletion inference, in which only the labels $(y_0,y_1)$ are revealed to the adversary, and $\rho$-insecurity against such attacks accordingly.
To contrast with instance and label deletion inference, we might use \emph{example} deletion inference attack to refer to our  default deletion inference attacks.
\end{definition}

Note that winning in an instance or label deletion inference game is potentially harder than winning the normal  variant (with full examples revealed to the adversary) as the adversary can always ignore the full information given to it. Hence, showing  successful instance deletion inference attacks is a stronger (negative) result. We empirically study the power of attacks in all these attack models.

\paragraph{Other variants of Definition~\ref{def:WeakDelPriv}.}
Definition~\ref{def:WeakDelPriv} can be seen as a \emph{weak} definition of privacy for deletion inference. The following list describe variants of Definition~\ref{def:WeakDelPriv} that are either directly weaker, or our attacks can be adapted to in a  rather straightforward way. 
\begin{itemize}
    \item {\bf Two-challenges vs. one challenge.} Definition~\ref{def:WeakDelPriv} includes two challenge examples and asks an adversary to find out which one is the actual deleted one. An alternative definition would only reveal one example to the adversary and asks it to tell if the example is deleted or not.\footnote{If one can sample from the set $\cS$ the two attack models can be shown to be  equivalent using standard hybrid arguments when the adversary's success probability is negligible in security parameter. This is similar to how a similar reduction works for CPA/CCA security games in cryptography.}
    \item {\bf Deletion-revealing vs. deletion-hiding.} Definition~\ref{def:WeakDelPriv} does not aim to hide the fact that a deletion has happened. An alternative definition could even aim to capture hiding the deletion itself by sampling the non-deleted   example \emph{outside} the dataset. A hybrid variant would challenge the adversary to distinguish between a deleted example versus a fresh sample from the distribution under which the learning is happening. All of our attacks apply to all these variants, but for brevity of presentation, we pick the deletion-revealing variant as the default.
    \item {\bf Random vs. chosen challenges.} Definition~\ref{def:WeakDelPriv} asks the adversary to distinguish between a \emph{random} pair of challenge examples, one of which is deleted. In a stronger attack model, the adversary is allowed to \emph{choose} the challenge examples.  
    \item {\bf Auxiliary information.} Definition~\ref{def:WeakDelPriv} does not explicitly give any extra information about other examples $e_k, k \notin \set{i,j}$ to the adversary, while a real-word adversary might have such knowledge.
    \item {\bf Multiple deletions vs. one deletion.} Definition~\ref{def:WeakDelPriv} does not allow more than one deletion to happen, while in general users might request multiple deletions to happen over time. {In fact, in Section~\ref{sec:compare}, we use this variant of the attacks to test our attacks on large data sets and compare the result with deletion inference attacks that are obtained by reduction to membership inference.}
\end{itemize}

In Section~\ref{sec:GGV}, we discuss stronger security definitions  that once satisfied would prevent the attack of Definition~\ref{def:WeakDelPriv} and all the variants above as special cases of the Deletion Compliance framework of  Gar et al. \cite{GGV20}. In particular, the definitions of this section  (including Definition~\ref{def:WeakDelPriv}) model \emph{weaker} security guarantees than that  of the Deletion Compliance framework of \cite{GGV20}, which makes our attack results of this section stronger.

\subsubsection{Reducing Deletion Inference to Membership Inference}
One can always reduce the task of  deletion inference to the task of membership inference. In particular, if we had a perfect membership inference oracle, we could use it to infer whether a given example is deleted or not by calling the membership inference oracle on the two models $\model,\modelDel$. 
%

 Algorithm~\ref{alg:reduction} below shows  an intuitive way to reduce deletion inference (DI) to \emph{imperfect} membership inference (MI) in a black-box way. Specifically, suppose the membership inference adversary $M(e, \model) \rightarrow \{0, 1\}$ returns $1$ if (it thinks) $e$ is a member of the dataset that is used to obtain the model $\model$. Then, if a deletion inference adversary  wants to find out whether $e  $ is deleted from the model $\model$ to reach the model $\modelDel$, it can simply run $M(e, \modelDel)$ and output what it outputs.  Note that there is no need to run $M(e, \model)$, as the adversary of Definition~\ref{def:WeakDelPriv} is given the promise that both $e_0,e_1$ \emph{are} members of the initial dataset $\cS$. Then the only question is how to combine the answers $M(e_0,\modelDel),M(e_1,\modelDel)$, which Algorithm~\ref{alg:reduction} decides in a natural way.

\begin{alg}[From membership to deletion inference] \label{alg:reduction} Given examples $e_0 = (x_0, y_0)$, $e_1 = (x_1, y_1)$ and models $\modelDel$, 
the reduction from deletion inference to membership inference proceeds as follows:
\begin{enumerate}
    \item Perform two membership inferences to obtain  $b_0 = M(e_0, \modelDel)$ and $b_1 = M(e_1, \modelDel)$.
    \item Return 0 if $b_0 = 0,b_1=1$, return 1 if $b_1 = 1,b_0=0$, and return a random bit if $b_0=b_1$. \qedhere
 \end{enumerate}
\label{const:reduction_adversary}
\end{alg}

\paragraph{Using confidence probabilities.} An alternative reduction to Algorithm~\ref{alg:reduction} can use the \emph{confidence} probabilities of $ M(e_0, \modelDel)$ and $ M(e_1, \modelDel)$ instead of their final (rounded) values. In this variant, the reduction returns $0$ if the confidence difference of $M(e_0, \modelDel) - M(e_0, \model)$ to output zero is more than the confidence difference of $M(e_1, \modelDel) - M(e_1, \model)$ to output zero. 

\subsection{Our Baseline Deletion Inference Attacks} 
We propose two variants of  attacks: (1) (example) deletion inference attack of $\DInfWith$ which uses both instances and their true labels, and (2) instance inference attack of $\DInfNo$ which only uses the  instances, without knowing the true labels. 
(In the next subsection, we also show how to \emph{find} the deleted label, which can be seen as a form of ``label reconstruction''and is stronger than  label inference attacks.)

\paragraph{Attack  $\DInfWith$  using labeled examples.} 
Our  example inference attack $\DInfWith$ is parameterized by a loss function $\loss$ and proceeds by first computing the loss for both examples $e_0,e_1$ on both models $\model,\modelDel$. Then, this attack identifies the deleted example by picking the example that leads to a \emph{larger increase} in its loss when we go from $\model$ to $\modelDel$. The intuition behind our attack is that the examples in the dataset are optimized (to a degree depending on the learning algorithm) to have small loss, while examples outside the dataset are not so. Therefore, once an example goes from inside the dataset to outside, it incurs a larger increase in loss.  We now define the attack formally. 

\begin{alg}[Attack $\DInfWith$] \label{const:delinflbl}
The attack is defined with respect to a loss function $\ell$. For any example $e$, we define the \emph{loss increase} of $e$ as:
$\delta(e,\model,\modelDel) = \loss(\modelDel, e) - \loss(\model, e).$
The adversary is given two labeled examples $e_0 = (x_0, y_0)$ and $e_1 = (x_1, y_1)$ and also has oracle access to $\model$ followed by access to $\modelDel$. 
The attack   proceeds as follows. 
\begin{enumerate}
    \item Query   $\model$ on both $x_0,x_1$.
    \item After getting access to $\modelDel$, query $\modelDel$ on both $x_0,x_1$.
    \item     Compute loss increases $\delta(e_0,\model,\modelDel)$ and $\delta(e_1,\model,\modelDel)$, and let $\alpha=\delta(e_0,\model,\modelDel) - \delta(e_1,\model,\modelDel)$.
    \item   Output $0$ if $\alpha>0$, output $1$ if $\alpha<0$, and output a uniformly random bit $b' \in \bits$ if $\alpha=0$. \qedhere
\end{enumerate}
\end{alg}

\paragraph{Connection to  memorization.} At a high level, $\DInfWith$ can be seen as   generalizing the notion of memorization by Feldman \cite{feldman2020does} from the 0-1 loss to general loss functions. More formally, if we use the 0-1 loss, then for $e \in \cS$, the \emph{expected} value $\Exp_{\modelDel \gets \Del(\model,e)}\delta(e,\model,\modelDel) $ would become equal to $\mathsf{mem}(\Learn, \cS, e)$   defined in \cite{feldman2020does}   to measure how much the learner $\Learn$ is memorizing the labels of its training set.   Using this intuition, our adversary picks the example that is \emph{most memorized} by the model.

The following lemma further formalizes the intuition behind our attack $\DInfWith$, so long as the the learning algorithm is the ERM rule.
\begin{lemma} \label{prop:changes}
Let $\ERM$ be the empirical risk minimization learning rule using a loss function $\loss$.
Let $\model=\ERM(\cS)$, $\modelD{e} \gets \Del(h,e)$ for $e \in \cS$, and  $\cS_{-e} = \cS \sm \set{e}$. Let $\delta_e = \delta(e,\model,\modelD{e})$, and let $\delta_{-e} = \Ex_{e' \gets \cS_{- e}} [\delta(e',\model,\modelD{e})]$ be the expected value of loss increase for examples that \emph{remain} in the dataset. Then   the following two hold.
\begin{enumerate}
    \item $\delta_{-e} \leq 0$.
    \item $\delta_e \geq -(n-1)\cdot \delta_{-e}$ where $n=|\cS|$. (In particular, by Part 1, it also holds that $\delta_e \geq 0$.)

\end{enumerate}
\end{lemma}
\begin{proof}
The first item of the lemma holds simply because we are using the ERM rule. Namely,   $\modelD{e}$   minimizes the empirical loss over $\cS_{-e} = \cS \sm \set{e}$. Therefore:
$$\delta_{-e} = \Risk_{\cS_{-e}}(\modelD{e}) -  \Risk_{\cS_{-e}}(\model) \leq 0.$$

Having proved the first part,  the second part  also follows due to using   the ERM rule. In particular, suppose for sake of contradiction that $ \delta_e < -(n-1)\cdot \delta_{-e}$, where $n=|\cS|$. Then,
$$\loss(\model, e) +(n-1) \cdot \Risk_{\cS_{-e}}(\model)> \loss(\modelD{e}, e) +(n-1) \cdot \Risk_{\cS_{-e}}(\modelD{e}). $$ 
Then, this implies 
\begin{equation*}
    \begin{split}
         \Risk_\cS(\model) &= \frac{\loss(\model,e) +(n-1) \cdot \Risk_{\cS_{-e}}(\model) }{n} \\ &>
\frac{\loss(\modelD{e},e) +(n-1) \cdot \Risk_{\cS_{-e}}(\modelD{e}) }{n} =
\Risk_\cS(\modelD{e}).
    \end{split}
\end{equation*}
However, the this  contradicts that the   $\ERM$ rule outputs   $h$ on training set $\cS$.
\end{proof}
Proposition~\ref{prop:changes} shows that whenever (1) $\delta_{-e} = \Ex_{e' \gets \cS_{-e}} [\delta_{e'}] < 0$ and (2) $\delta(e',\model,\modelD{e})$ for $e' \in \cS_{-e}$ is concentrated around its mean $\delta_{-e}$, then for a random $e' \in \cS_{-e}$, the attack $\DInfWith$ of Algorithm~\ref{const:delinflbl} would   likely   identify the deleted example correctly. Even though, in general we are not able to prove when these two conditions hold,   our experiments  confirm that these conditions indeed hold   in many natural scenarios, leading to the success of  $\DInfWith$ of Algorithm~\ref{const:delinflbl}.

\paragraph{Attack  $\DInfNo$  using instances only.} We now discuss our attack that does not rely on knowing the true labels $y_0,y_1$. 
The intuition is that, even if we do not know the true labels, when an example $e$ is deleted from the dataset, the change in the predicted label for $e$ is likely to be more than that of other examples that stay in the dataset. The reason is that for the remaining examples, the model is still trying to keep their prediction close to their correct value, but this optimization is not done for the deleted example $e$. Hence, our adversary would pick the candidate example that leads to \emph{larger change} in the output \emph{label} (not necessarily the loss). Hence, the attack is more natural to be used for regression tasks, even though it can also be used for classification if one uses the confidence parameters instead of the final labels.

\begin{alg}[Attack $\DInfNo$]
The attack is parameterized by a distance metric $\dis$  over $\cY$ (e.g., $\cY=\R$ and $\dis(y_0,y_1)=|y_0-y_1|$). The adversary is given two instances $x_0,x_1$, and it has oracle access to $\model$ followed by $\modelDel$. 
The attack then proceeds as follows.
\begin{enumerate}
    \item Query the models (in the order of accessing them) to get $\model(x_0)$, $\model(x_1)$, $\modelDel(x_0)$,  $\modelDel(x_1)$, and let $\beta = |\model(x_0) - \modelDel(x_0)| - |\model(x_1) - \modelDel(x_1)|$. 
    \item Return 0 if $\beta>0$,  return 1  if $\beta<0$, and return a random answer in $\bits$ if $\beta=0$. \qedhere
\end{enumerate}
\label{const:delinf}
\end{alg}


\subsection{Experiments: Deletion Inference Attack on Regression}
\label{sec:del_inf_regression}

Now we apply our attack $\DInfWith$ (Algorithm~\ref{const:delinflbl}) and attack $\DInfNo$ (Algorithm~\ref{const:delinf})   on multiple regression models including Linear Regression, Lasso regression, SVM Regressor, Decision Tree Regressor, and Neural Network Regressor\footnote{Implementation of the methods are from the python library Scikit-learn.}. Details of the attacked models are included in Appendix~\ref{sec:hyperparameter}. 

\paragraph{Experiment details.} Table~\ref{tab:dataset} includes the details of all the datasets we used in the deletion inference experiments and also in other experiments later. We use two regression datasets Boston and Diabetes. For training the original model $h$, we use a random subset with 90\% of the dataset.
The experiment follows the security game of Definition~\ref{def:WeakDelPriv}. 
To ensure the perfect deletion,  $\modelDel$ is obtained by a  full re-training   with the  dataset without the deleted example.
For the attack $\DInfWith$, we use squared loss, which is defined as $\ell(\model, (x, y)) = (\model(x) - y)^2$.
Finally, we repeat the security game of Definition~\ref{def:WeakDelPriv} 1000 times and take the average success probability of the adversaries.

\paragraph{Results.} The result is shown in Table~\ref{tab:deletion_reg}. In most cases, our adversary gets more than 90\% success probability in the deletion inference.


\begin{table*}[t]
\resizebox{0.98\textwidth}{!}{%
\centering
\begin{tabular}{llllll}
\hline
                                &                                     & No. Samples & No. Features & Label           & Predict                       \\ \hline
\multirow{2}{*}{Regression}     & Boston \cite{harrison1978hedonic}                     & 506         & 14           & Real       & The median house price        \\
\cline{2-6}
                                & Diabetes \cite{efron2004least}                             & 442         & 10           & Real       & Disease progression   \\ 
\hline
\multirow{3}{*}{Classification} & Iris \cite{fisher1936use}            & 150         & 4            & 3 types  & The type of iris plants       \\ \cline{2-6}
                                & Wine \cite{aeberhard1994comparative} & 178         & 13           & 3 types & Wine cultivator               \\ \cline{2-6}
                                & Breast Cancer \cite{street1993nuclear}                       & 569         & 30           & Binary          & Benign/malignant  tumors
                                \\ \cline{2-6} & 1/12MNIST\cite{lecun1998gradient} & 5000 & 784 & 10 types & Digit between 0 to 9  
                                \\  \cline{2-6} &  CIFAR-10 \cite{krizhevsky2009learning} & 60000 & 3072 & 10 classes & Image classification\\ \cline{2-6} & CIFAR-100 \cite{krizhevsky2009learning} & 60000 & 3072 & 100 classes & Image classification
               \\ \hline
\end{tabular}
}
\captionsetup{justification=centering}
\caption{Descriptions of the datasets used in deletion inference.}
\label{tab:dataset}
\end{table*}

\begin{table}[ht]
\centering
\begin{tabular}{l|ll|ll}
\hline 
                    & \multicolumn{2}{c|}{Boston} & \multicolumn{2}{c}{Diabetes} \\ \hline
                    \textit{Learning Method} & $\DInfWith$  & $\DInfNo$ & $\DInfWith$  & $\DInfNo$  \\ \hline
Linear regression & 99.8\%  & 99.1\%  & 99.8\%  & 99.3\%  \\ \hline
SVM               & 93.9\%  & 89.1\%  & 99.2\%  & 100.0\% \\ \hline
Lasso regression  & 98.8\%  & 97.1\%  & 99.3\%  & 98.3\%  \\ \hline
Decision tree     & 100.0\% & 100.0\% & 100.0\% & 100.0\% \\ \hline
MLP  & 80.4\% & 78.3\% & 72.2\% & 72.3\% \\
\hline
\end{tabular}
\caption{Success probabilities  of various attacks on regressors for different datasets.}
\label{tab:deletion_reg}
\end{table}

\subsection{Experiments: Deletion Inference Attacks on Classification}
\label{sec:del_inf_classification}

In this experiment, we apply $\DInfWith$ and $\DInfNo$ on classification tasks. In our experiments, we use different models, including logistic regression, support vector machine (SVM), Decision tree, random forest, and multi-layer perceptron (MLP).   Due to page limit, details of the models are included in Appendix~\ref{sec:hyperparameter}.

\paragraph{Experiment details.} 
 We use datasets Iris, Wine, Breast Cancer, and 1/12MNIST. (The details of the datasets are shown in Table~\ref{tab:dataset}.) Similarly to attacks on regression, We pick a random 90\% fraction of the dataset to train the model, and we do a full retrain to obtain $\modelDel$.
The difference compared to the case of regression  is that the label space $\cY$ is now a finite set. 
In this experiment, we assume the output of any hypothesis function $h \in \cH$  is a multinomial (confidence) distribution over $\cY$, and this probability is available to the adversary. This assumption is realistic as many machine learning applications have the confidence as part of the output~\cite{ribeiro2015mlaas}, and this is also the default setting of many adversarial machine learning researches~\cite{shokri2017membership, long2018understanding}\footnote{The model in this scenario is still considered as black-box in most machine learning adversarial literature, but someone may argue it is not fully black-box.}. To formally fit the attack into the framework of Definition~\ref{def:WeakDelPriv}, we can extend the set $\cY$ to directly include any such multinomial distribution as the actual output ``label''.

For $\DInfWith$, we use the negative log likelihood   loss function  $\ell(h, (x, y)) = -\log\left(\Pr[h(x) = y]\right)$. We then repeat the security game of Definition~\ref{def:WeakDelPriv} 1000 times to approximate the winning probability.

\paragraph{Results.} 
 We present the   result of attacks $\DInfWith$ and $\DInfNo$ on three classification datasets in Table~\ref{tab:deletion_class}. As anticipated, the success rates $\DInfWith$ are noticeably larger than those of $\DInfNo$.


\begin{table*}[ht]
\centering
\resizebox{\textwidth}{!}{%
\begin{tabular}{l|ll|ll|ll|ll}
\hline 
                   Datasets $\rightarrow$ & \multicolumn{2}{c|}{Iris} & \multicolumn{2}{c}{Wine} & \multicolumn{2}{|c}{Breast Cancer} & \multicolumn{2}{|c}{1/12 MNIST} \\ \hline
                    \textit{Learning Method $\downarrow$} & $\DInfWith$  & $\DInfNo$ & $\DInfWith$  & $\DInfNo$ & $\DInfWith$      & $\DInfNo$ & $\DInfWith$      & $\DInfNo$      \\ \hline
Logistic Regression & 88.3\%      & 86.8\%   & 80.8\%      & 76.1\%   & 69.1\%          & 60.6\%  & 72.9\% & 56.6\%     \\ \hline
Decision Tree       & 100.0\%     & 100.0\%  & 100.0\%     & 100.0\%  & 100.0\%         & 100.0\%   & 100.0\% & 100.0\%     \\ \hline
SVM                 & 70.5\%      & 60.3\%   & 76.9\%      & 66.7\%   & 73.8\%          & 57.3\%  & 72.3\% & 62.0\%  \\ \hline
Random Forest       & 89.2\%      & 89.1\%   & 83.3\%      & 78.1\%   & 89.2\%          & 85.7\%  & 89.9\% & 84.5\% \\ \hline
MLP & 92.9\% & 55.5\% & 54.2\% & 51.1\% & 83.5\% & 67.7\% & 62.5\% & 59.0\%  \\
\hline
\end{tabular}
}
\captionsetup{justification=centering}
\caption{Success probabilities  of the attacks $\DInfWith$ and $\DInfNo$  on classifiers.}
\label{tab:deletion_class}
\end{table*}

{ 
\subsection{Experiments: Attacking  Large  Datasets and Models} \label{sec:compare}
In this section, we aim to show that our deletion inference attacks can be scaled to work with large datasets and models. We demonstrate the power of our attacks on datasets of the same size as those of \cite{shokri2017membership} and compare the power of our direct deletion inference to doing reduction to the membership inference attack of \cite{shokri2017membership}.  We  show that using our method  can lead to \emph{significantly} stronger results than making a \emph{black-box} use of  membership inference attacks. 
}

We evaluate our deletion inference attacks $\DInfWith$ and $\DInfNo$ on large dataset and large neural networks. In our experiment, we use CIFAR-10 and CIFAR-100 datasets \cite{krizhevsky2009learning} as the training dataset, which are standard datasets for the evaluation of image classifiers, especially for deep learning models.

To better compare the success of our attacks with  \cite{shokri2017membership} we use a variant attack of Definition~\ref{def:WeakDelPriv} in which \emph{multiple} deletions happen (as explained in one of the variants following Definition~\ref{def:WeakDelPriv}). One advantage of this experiment setting is that the attack of  \cite{shokri2017membership} needs to train ``attack models'' for each victim model, and hence having multiple different deletions lead to multiple full training of attack models for \cite{shokri2017membership} which is very expensive to run. However, in the multiple-deletion attack setting, one needs to only train the attack models of \cite{shokri2017membership} twice to compare   our attack with a reduction to MI of \cite{shokri2017membership}.

\paragraph{Setting of our attack.} The success probability is then calculated by taking the average over 20 rounds of full experiment. In each round of experiment, we first train a deep model with $n$ examples,
where $n$ varies from $15,000, 20,000, 25,000$, and $29,540$ ($29,540$ is picked to match the scenario of \cite{shokri2017membership}). We then randomly remove a batch of 100 examples in the training dataset, and train a new model without those $100$ examples. As a reference, we pick another $100$ random examples that remains in the dataset. The success probability is calculated over every pairs (in total, $10,000$ pairs) of the deleted and reference examples, i.e., one deleted examples and one remaining example is given to the deletion inference adversaries $\DInfWith$ and $\DInfNo$. We then measure the fraction of all pairs in which our adversary correctly predicts the deleted example. We evaluate our results on two deep neural network models: 1. A convolutional neural network that includes two convolutional layers (called smallCNN below), similar to the network used in \cite{shokri2017membership}. 2. VGG-19 network (called VGG below) that has 19 layers in total, which is well-known for its power for image classification tasks.


\paragraph{Baseline settings for comparison.} We compare our attacks  with reductions to the membership inference attack in \cite{shokri2017membership}\footnote{We implemented \cite{shokri2017membership} attack. \cite{shokri2017membership} reports their membership inference attack achieves $71\%$ success rate on a CNN model with two convolutional layers that is trained with CIFAR-10 dataset with $15,000$ random examples. Our implementation of membership inference attack achieves 74\% success rate on smallCNN model (which also has two convolutional layers) and 88\% success rate on VGG model, which are trained on a subset of CIFAR-10 dataset with $15,000$ random examples. The success rate matches the number reported in their work.}, i.e., reduction with label only and reduction with confidence probabilities. 
 

\begin{figure*}[ht]
    \centering
\subfigure[]{\includegraphics[width = 0.48\textwidth]{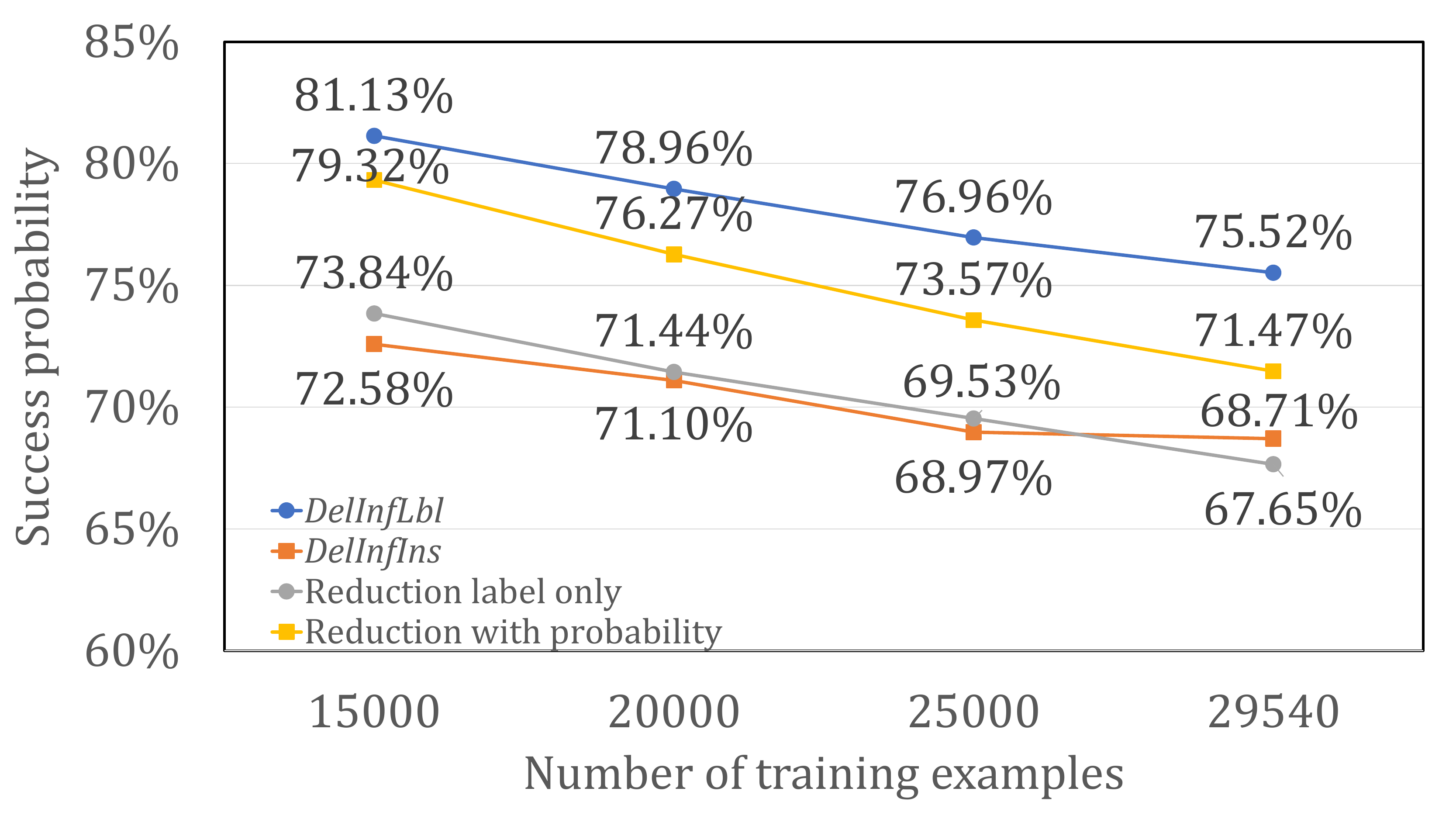}}
\subfigure[]{\includegraphics[width = 0.48\textwidth]{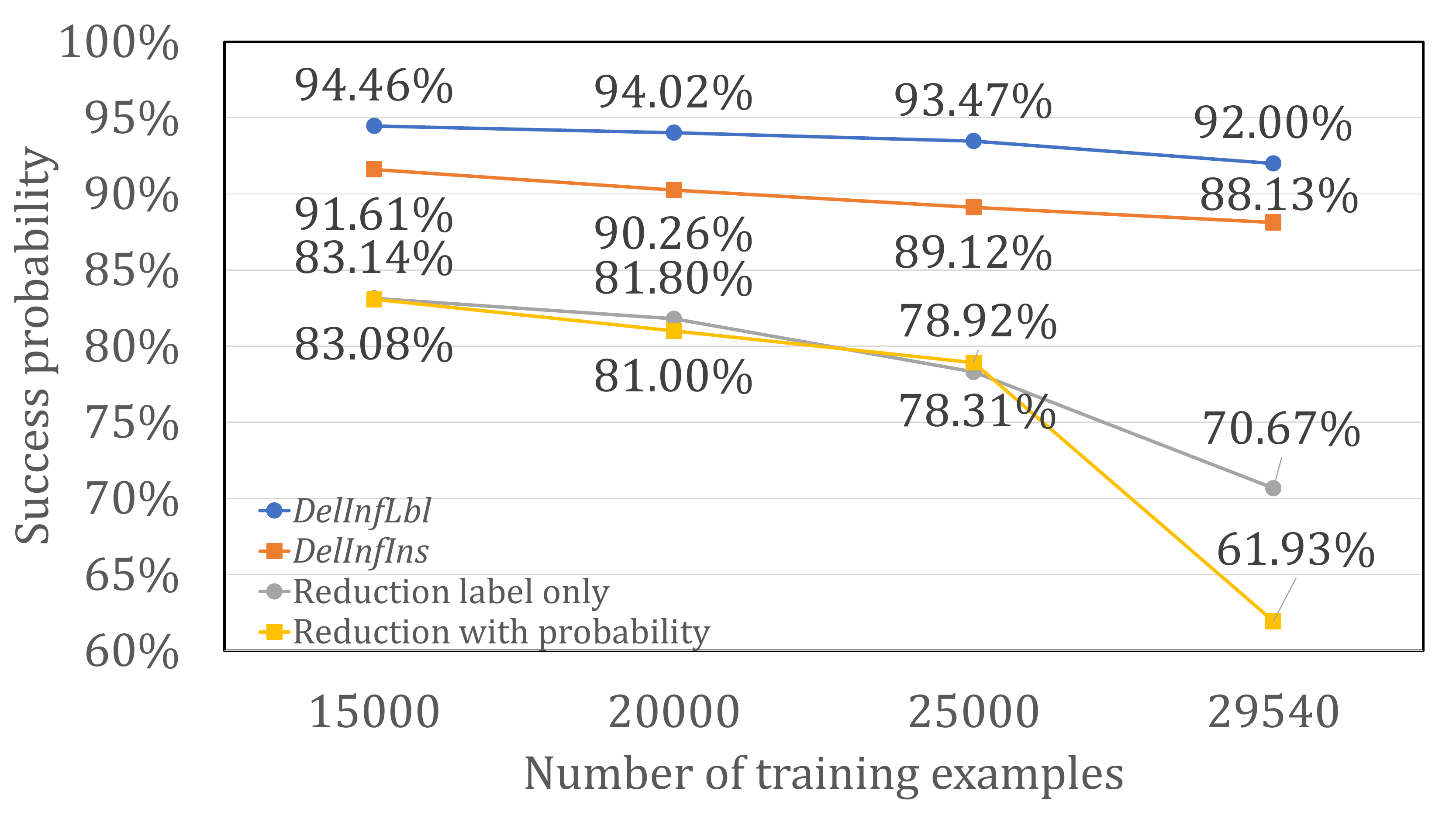}}
    \caption{{Trend of success probabilities of attacks $\DInfWith$ and $\DInfNo$ on \textbf{smallCNN} models trained with different number of examples are shown; (a) uses dataset CIFAR-10 and (b) uses dataset CIFAR-100 dataset. The success probabilities are also compared with two baseline attacks that are obtained by reductions to the membership inference attack of \cite{shokri2017membership}.}}
    \label{fig:dataset_size_cnn}
\end{figure*}

\paragraph{Results.} In Figure~\ref{fig:dataset_size_cnn} and~\ref{fig:dataset_size_vgg}, we analyze the success probabilities of our deletion inference adversaries $\DInfWith$ and $\DInfNo$ on smallCNN model and VGG model. Our attack is able to correctly predict most of the deletions in the deep learning models, even when a batch of examples is deleted at the same time. Furthermore, note that for the membership inference attack of \cite{shokri2017membership} to work, the adversary needs to have the label of the target instance and also make many queries to the target model for training an attack model (or many auxiliary data examples to train a similar model). On the other hand, our attack is extremely simple, and $\DInfNo$ even does not require the label of the example.

\begin{figure*}[ht]
    \centering
    \subfigure[]{\includegraphics[width = 0.48\textwidth]{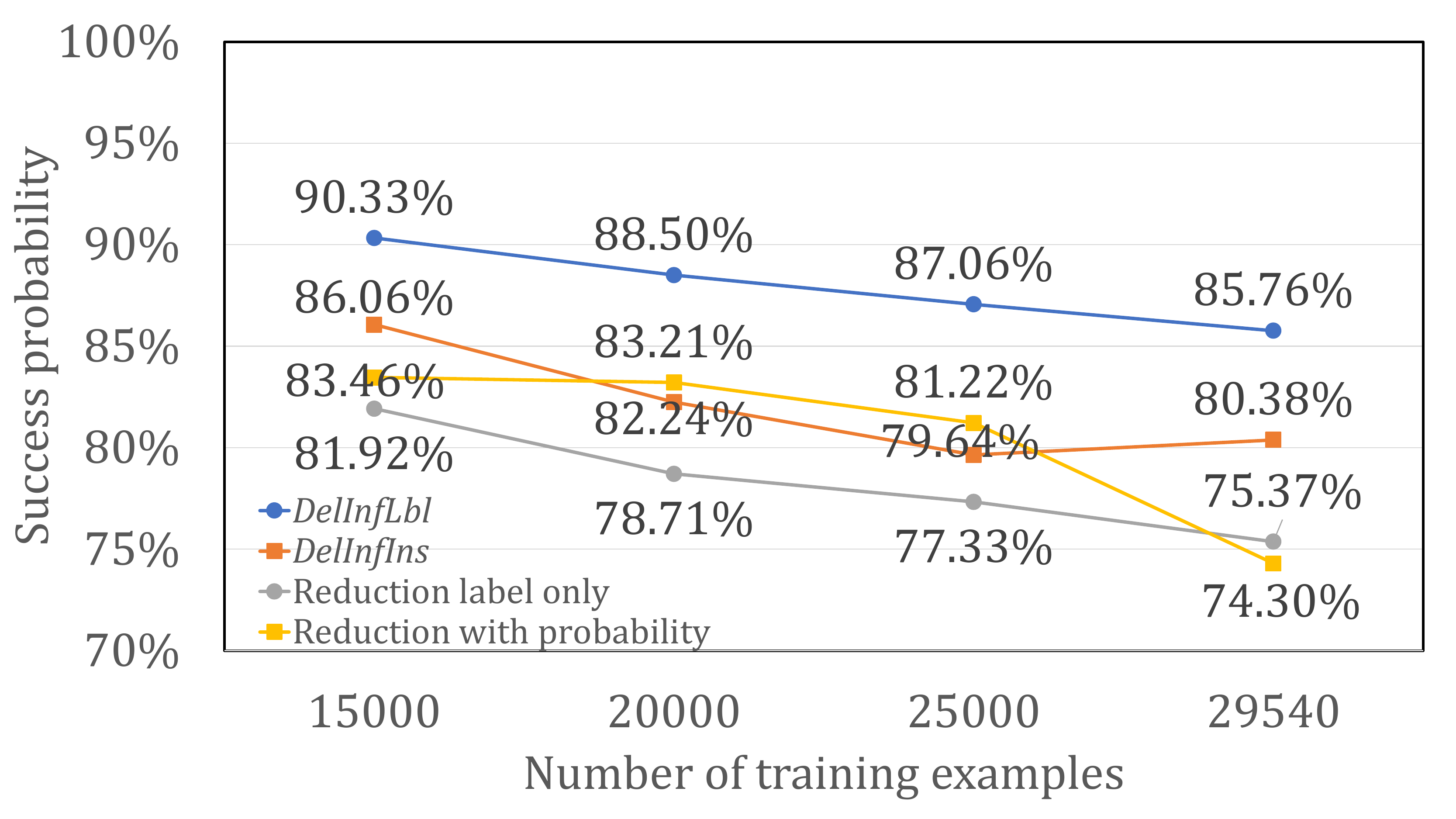}}
    \subfigure[]{\includegraphics[width = 0.48\textwidth]{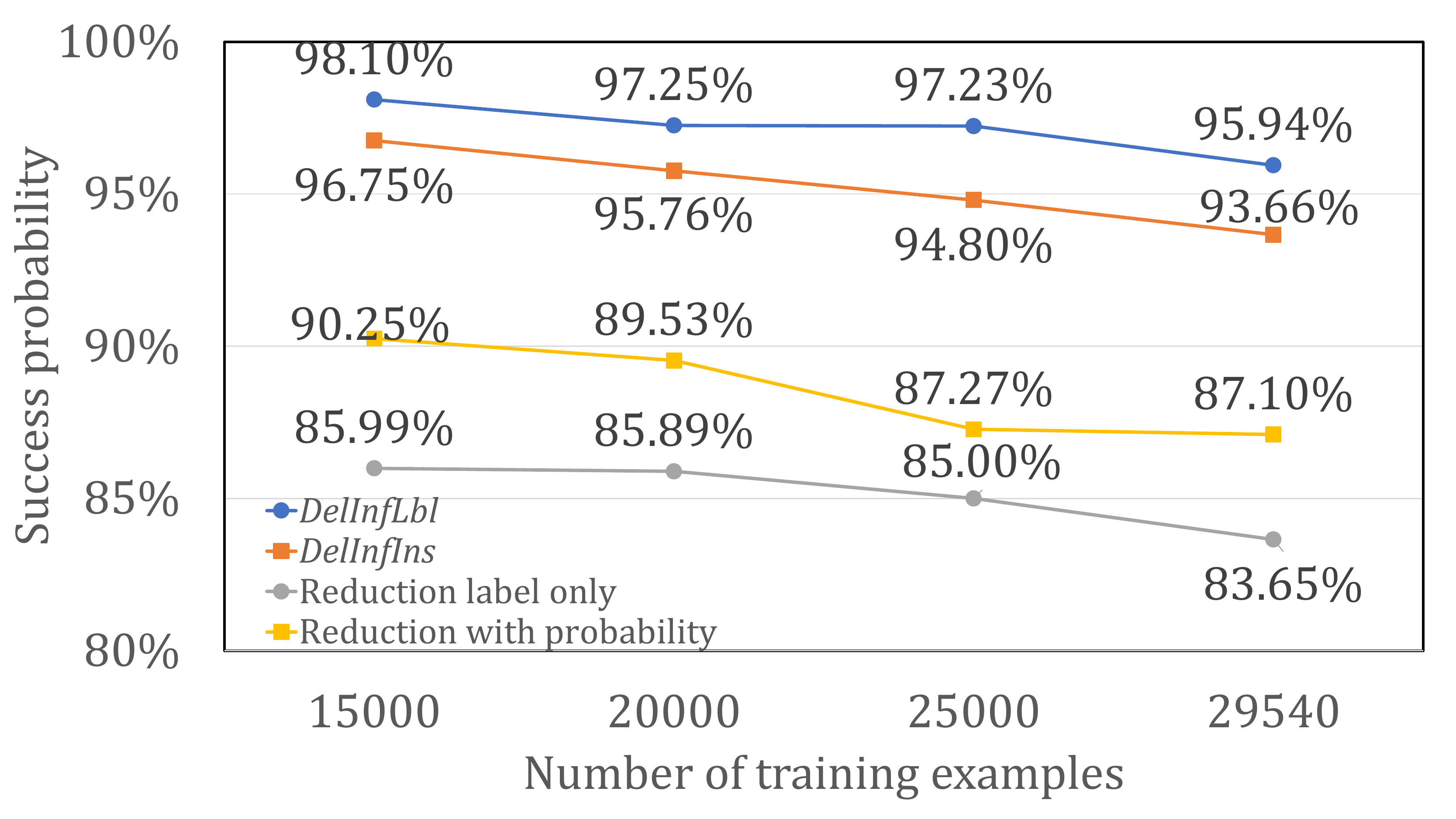}} 
    \caption{{Trend of success probabilities of attacks $\DInfWith$ and $\DInfNo$ on \textbf{VGG} models trained with different number of examples are shown; (a) uses dataset CIFAR-10 and (b) uses dataset CIFAR-100 dataset. The success probabilities are also compared with two baseline attacks that are obtained by reductions to the membership inference attack of \cite{shokri2017membership}.}}
    \label{fig:dataset_size_vgg}
\end{figure*}




\begin{remark}[About using reduction to MI as baseline]
Here we comment on the limitations of membership inference as a baseline attack, as membership inference is not tuned to distinguishing between two points (one of which is guaranteed to be in the training set). Indeed, membership inference attackers only get only one instance as input, while our formalization of deletion inference gets two inputs. However, please note that we compare our deletion inference attackers to \emph{reductions} to membership inference adversaries. The reduction is allowed to call the MI adversary \emph{multiple} times. Indeed our reduction of the previous subsection calls the MI adversary twice, and this change makes the \emph{reduction} to MI (which is a DI adversary itself)  powerful enough to be able to win the DI inference game with probability close to 1, so long as its (regular) MI oracle   wins its own game with probability close to 1.
\end{remark}

\section{Deletion Reconstruction} \label{sec:del-Rec}
Section~\ref{sec:del-infer} focused on attacks that infer which   of the two given examples is the deleted one. A more devastating form of attack aims to \emph{reconstruct} the deleted example by querying the two models (before and after deletion). In this section, we show how to design such stronger attacks.
We  propose two types of reconstruction attacks on the deleted example. The first one focuses on reconstructing the deleted instance, while the second one focuses on reconstructing the  deleted label. Both types of attacks follow the same security game which is explained in the definition below.

\subsection{Threat Model}
\begin{definition}[Deletion reconstruction attacks] \label{def:Del-Recon}
Let $\Learn$ be a learning algorithm, $\Del$ be a deletion mechanism for $\Learn$, and $S_n$ be a distribution over $(\cX \times \cY)^n$.
Consider the following game played between the adversary $\Adv$ and challenger $\Chal$.
\begin{enumerate}
    \item {\bf Sampling the data and random selection.} $\Chal$ picks a dataset $\{e_1 \dots e_n\} = \cS \gets S_n$ of size $n$. 
    It also chooses $i \gets [n]$ at random. 
        \item \label{Step:Chal} {\bf Oracle access  before deletion.}  The challenger $\Chal$   trains $\model \gets \Learn(\cS)$.
  The adversary $\Adv$ is then given \emph{oracle} access to $\model$. At the end of this step, the adversary instructs moving to the next step.
  
      \item {\bf Deletion.} The challenger obtains $\model_{-e_i} \gets \Del(h,e_i)$.
     \item {\bf Oracle access after deletion.}     
        The adversary $\Adv$   is now given (only) \emph{oracle} access to  $\model_{-e_i}$.
    
    \item {\bf Adversary's guess.} Adversary outputs a guess $e$. 
\end{enumerate}
 For a similarity metric  $\dis$ defined on $(\cX \times \cY)$, the adversary $\Adv$ is called a  \emph{$(\rho,\eps)$-successful deletion reconstruction} attack  if it holds that $\Pr[\dis(e,e_i) \leq \eps ]\geq \rho$. 
For bounded $\dis(\cdot,\cdot) \in [0,1]$ and an adversary $\Adv$, we define the  \emph{expected accuracy} of $\Adv$ as $1-\Ex[\dis(e,e_i)]$.
\end{definition}

\paragraph{Limited reconstruction attacks.} One can   use Definition~\ref{def:Del-Recon} to capture attacks in which  the goal of the adversary is to only (perhaps partially) reconstruct the instance $x$ or the label $y$. In  case of approximating $x$, we can use a metric distance $\dis$ that is only defined over $\cX$ and ignores the labels of $e$ and $e_i$. We refer to such attacks as \emph{deleted instance reconstruction} attacks. Similarly, by using a proper metric distance defined only over $\cY$, we can use Definition~\ref{def:Del-Recon} to obtain \emph{deleted label reconstruction} attacks. Finally, to \emph{completely} find $e$ (resp. $x$ or $y$)  we  use the 0-1   metric  $\dis(e,e')=\one[e \neq e']$ (resp. $\one[x \neq x']$ or $\one[y\neq y']$). 

One can also observe that Deletion Inference   can generically be reduced to Deletion Reconstruction.

\begin{theorem}[From reconstruction to inference] \label{thm:fromRecToInf}
Let $\Learn$ be a learning algorithm, $\Del$ be a deletion mechanism for $\Learn$,     $\dis$ be a distance metric over $(\cX \times \cY)$, and   $S_n$ be a distribution over $(\cX \times \cY)^n$. Suppose  there is a  $(\rho,\eps)$-successful PPT  reconstruction adversary against the scheme $(\Learn,\Del)$, and  
  $\Pr[\dis(e_0,e_1) > 2\eps] \geq 1-\delta$ where the probability is over sampling $e_0,e_1$ from the sampled dataset $\cS \gets S_n$.\footnote{For example, when $S_n$ consists of $n$ i.i.d. samples from $D$, $e_0,e_1$ are simply two independent samples from $D$.}
 Then,  $(\Learn,\Del)$ is $(\rho-\delta)$-insecure against deletion inference   over distribution $S_n$.
\end{theorem}

\begin{proof} 
We give a polynomial time reduction. In particular, suppose $B$ is a (black-box) adversary that shows the  $(\rho,\eps)$ insecurity of the scheme $(\Learn,\Del)$ against deletion  reconstruction attacks. We design an adversary $\Adv$ against deletion inference (as in Definition~\ref{def:WeakDelPriv}) as follows. Given $(e_0,e_1)$  as challenges, first ignore $(e_0,e_1)$ and using oracle access to models $\model,\modelDel$, run $B$ to obtain $e$ as approximation of the deleted example. Output $0$ if $\dis(e_0,e)\leq \eps$,  else  output $1$ if $\dis(e_1,e)\leq \eps$,  otherwise output   uniformly  in $\bits$.

We now analyze the reduction above. With probability at least $\rho$ over the execution of the attack $B$, it holds that $\dis(e,e_b) \leq \eps$, where $e_b$ is the  deleted example. Also, with probability $1-\delta$ it holds that $\dis(e_0,e_1)>2\eps$. By a union bound, we have that with probability  at least $\rho-\delta$ both of the conditions above happen at the same time, in which case the adversary $\Adv$ outputs the correct answer $b$. 
\end{proof}

%
%
Due to the theorem above, all the reconstruction attacks below can be seen as strengthening of deletion inference attacks.

\subsection{ Deletion Reconstruction of Instances for Nearest Neighbor} \label{sec:Singleton}
In this experiment, we consider a classification clustering task in high dimension.
The previous work \cite{feldman2020does,brown2021memorization} studied the same setting and showed that machine learning models sometimes need to memorize their training set in order to learn with high accuracy. In this setting, we extend the attacks of \cite{feldman2020does,brown2021memorization} into two directions to obtain deletion reconstruction attacks: (1) we obtain polynomial time  attacks that extract instances rather than proving mutual information between the model and the examples, (2) we show a setting where the extraction is enabled after the \emph{deletion}.  

\paragraph{Roadmap and the leakage of the deletion.} We develop polynomial-time reconstruction attacks that crucially leverage the deletion operation. However, in order to analyze our attacks,  we first \emph{limit} ourselves to the so-called \emph{singleton} setting in which each label appears at most once for an example in the dataset (Section~\ref{sec:singletonTheory}). Focusing on this case allows us to provide theoretical ideas that support our attacks.  However, our  attacks in the singleton case are also able to extract instances \emph{even} without deletion. Hence, in the singleton case, our attacks can be seen as leakage of the model $\model$ itself, \emph{even without deletion}. Note that such attacks can still be used for deletion reconstruction, they do not reflect the \emph{extra leakage} of the deletion operation. Nevertheless, we next experimentally show (Section~\ref{sec:visual}) that virtually the same polynomial-time attacks succeed even when the labels are not unique on the real world dataset Omniglot. In particular, when we have many repeated labels (perhaps even as neighbor cells), then our simpole attacks do \emph{not} extract the instances from access to either of $\model,\modelDel$, and it is needed to have access to \emph{both} models to find the ``vanished'' Voronoi cell before extracting the center of the cell.


We now our polynomial-time deletion reconstruction attack for the case of 1-nearest neighbor models. We work with instance space $\cX=\{0, 1\}^d$.\footnote{We use binary features because it is more general and that other features can also be represented in the form of binary strings.} We also assume the learner $\Learn$ runs a $1$-nearest neighbor algorithm. Namely for $\model  = \Learn(\cS)$ where $\cS = \{(x_1, y_1) \dots (x_n, y_n)\}$, we have $\model(x) = y_j$ where $j= \argmin_{i} \dis(x, x_i)$. 

We propose the following attack $\DelRecon$ that aims to reconstruct the deleted instance $x_i$. 

\begin{alg}[Attack $\DelRecon$]  \label{alg:recon_general}
Suppose the adversary is given oracle access to $\model$ followed by oracle access to $\modelDel$, along with an auxiliary set of instances $\cT$, $|\cT|=m$. (For example, $\cT$ could simply be $m$ independent samples different from the original training set $\cS$.) The attack then proceeds as follows:
\begin{itemize}
    \item For all $x \in \cT$ query the model $\model$. 
    \item Then for all $x \in \cT$, query the model $\modelDel$. 
    \item Create the set of points in the ``deleted region'': $\cT' = \left\{ x \mid \model(x) \neq \modelDel(x), x \in \cT \right\}$.
    \item Return the majority for each coordinate; namely, return $x = (b'_1, \dots, b'_d)$, where  $\forall i \in [d]$, 
    $$b'_i = \argmax_{b \in \{0, 1\}} \sum_{(b_1, \dots, b_d) \in \cT'} \one [b_i = b].$$ 
\end{itemize}
\end{alg}

\paragraph{Intuition behind the attack.}  The   intuition behind the attack of Algorithm~\ref{alg:recon_general} is that instances like $x$ whose prediction label changes during the deletion process should belong to the Voronoi cell centered at $x_i$, where $(x_i,y_i)$ is the deleted example. Then the algorithm heuristically assumes that when we pick $x$ at random \emph{conditioned} on changed labels, then they give a pseudo-random distribution inside the Voronoi cell of $x_i$.
In the next section we show that for a natural case called singletons, in which the labels are unique, this intuition carries over formally. We then experimentally verify our attack for the general case (when labels can repeat) on a real data set.



\subsubsection{Theoretical Analysis for Uniform Singletons} \label{sec:singletonTheory}
In this section, we focus on a theoretically natural case to analyze the attack of Algorithm~\ref{alg:recon_general}. We refer to this case as the \emph{uniform singletons} which is also studied in~\cite{feldman2020does,brown2021memorization}   and is as follows. First, we assume that instances are uniformly distributed in $\bits^d$, and secondly, we assume that the labels are unique (i.e., without loss of generality, the labels $y_1,\dots,y_n$ are just $1,\dots,n$). The following lemma shows that in this case,  Algorithm~\ref{alg:recon_general} never converges to wrong answers for \emph{any} coordinate of the instances.

\begin{lemma}[Non-negative correlations] \label{lem:nonNegative}
Let  $\cS = \{x_1,\dots,x_n\}$ where $\forall i, x_i \in \bits^d$, and suppose  $\model(x) = \argmin_{i} \dis(x, x_i)$, and we break ties by outputting the smallest index $i$, if multiple nearest neighbors exist. Suppose $\cC_i = \set{x \mid h(x) = i}$ be the Voronoi cell centered at $x_i$. Let $x[j]$ be the $j$'th bit of $x$. Then, for every $i \in [n]$ and every $j \in [d]$, we have
$$\Pr_{x \gets \cC_i}[x[j] = x_i[j]] \geq \frac{1}{2}. $$
\end{lemma}

\begin{proof}[Proof of Lemma~\ref{lem:nonNegative}]
Let $\cC^{j,b}_i = \set{x \in \cC_i \mid x[j]=b}$ be the subset of $ \cC_i$ that has $b$ in its $j$'th coordinate. 

We claim that by flipping the $j$'th bit of every $x \in \cC^{j,1-x_i[j]}$, we obtain a vector $x' \in \cC^{j,x_i[j]}$. The reason is as follows. (1)  By definition, the $j$'th bit of $x'$ is indeed $x_i[j]$. (2) It holds that $h(x')=i$, which means $x' \in C_i$. The reason for (2) is that, by flipping the $j$'th bit of $x$, $x'$ gets one step \emph{closer} to $x_i$ compared to how far $x$ was from $x_i$. Therefore, if $x_i$ was the nearest neighbor of $x$, it would also be the nearest neighbor of $x'$ as well. A boundary case occurs if multiple points are the nearest points of $x$, but the same tie breaking rule still assigns $x_i$ as the nearest neighbor of $x'$.
Since the mapping from $x$ to $x'$ is injective, it also gives an injective mapping from $ \left|\cC^{j,1-x_i[j]}_i\right|$ to $\left|\cC^{j,x_i[j]}_i\right|$. This proves that 
$$\left|\cC^{j,x_i[j]}_i\right| \geq \left|\cC^{j,1-x_i[j]}_i\right|,$$
which is equivalent to $\Pr_{x \gets \cC_i}[x[j] = x_i[j]] \geq 1/2$.
\end{proof}

{
\subsubsection{Experiments: Deleted Image Reconstruction for 1-NN} \label{sec:visual}
  We now show that the simple attack of Algorithm~\ref{alg:recon_general} can be used to reconstruct visually recognizable images even when the distribution is not normal and labels are not unique. Hence, we conclude that the actual power of this attack goes beyond the theoretical analysis of the previous section.
  We use the Omniglot~\cite{lake2015human} dataset, a symbol classification dataset specialized for few-shot learning. The dataset includes handwritten symbols from multiple languages.

\paragraph{Experiment details.} We binarize each pixel of the dataset to remove the noise in gray-scale. The input space is $\cX = \{0, 1\}^d$, where $d = 11025$ is the number of pixels.
{\color{black}{}  We assume the Omniglot dataset 
is divided into two parts: (1) a training subset which contains $140$ symbols from  $30$ different languages. The languages serve  as the class label in the dataset in our experiments,\footnote{Note that in the original dataset, the labels reflect the character, but to demonstrate the leakage of \emph{deletion} rather than the mere leakage of datasets alone, we use the labels that represent the languages to increase the frequency of the labels.} and (2) a fixed test set with another $140$ examples from each language which is provided to the adversary as auxiliary information.
The learning algorithm $\Learn$ is the $1$-nearest neighbor predictor, which for a dataset $\cS$ always returns the label (i.e., the language) of the nearest example in the dataset $\model(x') = \argmin_y\{ \dis(x, x') \mid (x, y) \in \cS\}$. We use Algorithm~\ref{alg:recon_general} as the attack, which simply takes majority on each pixel over the instances that fall into the disagreement region of the two models (before and after deletion). 
%
We run the security game of Definition~\ref{def:Del-Recon} with $100$ random images from the dataset as the deleted image. 

\paragraph{Comparison with reconstruction attacks  without deletion.} As a comparison to further highlight the leakage that happens due to the deletion, we also run a similar reconstruction attack \emph{without} deletion. Suppose for a moment that labels were unique. Then, to reconstruct instance $x$, the attacker aims to extract the image $x$ from the data set with label $y$, where $y$ is the label of $x$. To do that, the reconstruction attacker can run the same exact attack as our deletion reconstruction, as follows: it tests all the images in the test dataset on the model and records every image with label $y$. The attacker then generate a reconstruction image by taking the majority of the images with label $y$ on every pixel. 

When the labels are unique, this reconstruction attack can reconstruct the instances used by a 1-NN just like how our deletion reconstruction attack does and succeeds.   However, in our case labels are not unique. Hence, we use this attack as the baseline to show how much our deletion inference attack is in fact extracting information that is the result of the deletion operation.

The result of our deletion reconstruction and the baseline (non-deletion) reconstruction attacks are  shown in Figure~\ref{fig:knn_new}. Our deletion reconstruction algorithm reconstructs $40$ out of $100$ images, due to page limit, $33$ of them is shown in Figure~\ref{fig:knn_new}. As is clear from the pictures, the non-deletion reconstruction attack gives no meaningful result in our setting. { More concretely, for $35$ of the $40$ images the label of the deletion reconstruction attack obtains the the correct label when fed back into the nearest neighbor classifier, while only $1$ of the images generated by the attack without deletion obtains the correct label.}





\begin{figure}
    \centering
    \includegraphics[width = 0.48\textwidth]{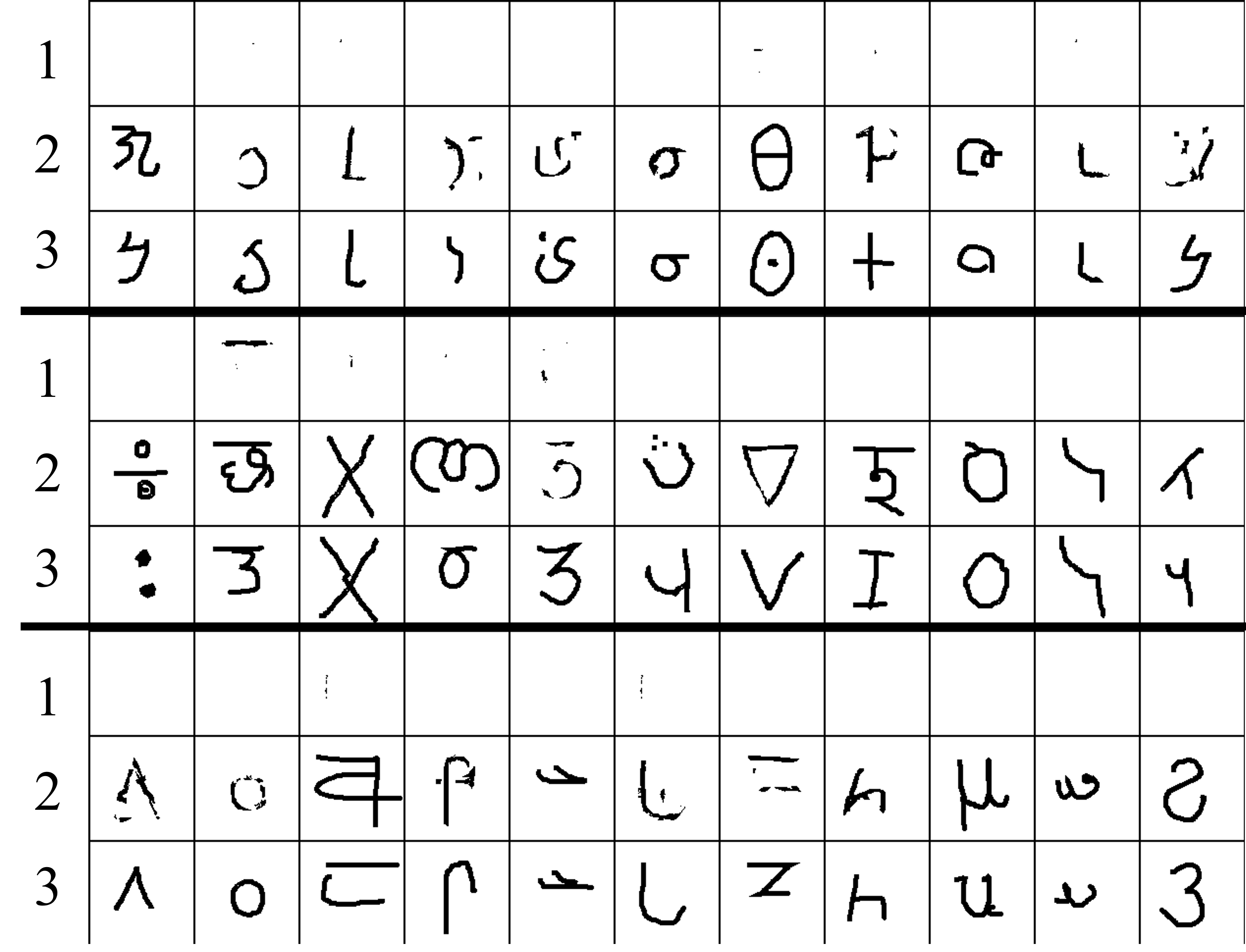}
    \caption{$33$ reconstruction examples on Omniglot dataset. In the figure, Row 1 is the result from the attack without deletion, Row 2 is the result from the deletion reconstruction attack, and Row 3 is the deleted example, which is the target of the attack.}
    \label{fig:knn_new}
\end{figure}
}
}

\subsubsection{Experiments:  Deleted Sentence Reconstruction for  Language Models}

\label{sec:recon_language}
  In this experiment we perform  {reconstruction attacks} on sequential text data. Namely, we show how to extract the deleted sentence by querying a language model according to the security game of Definition \ref{def:Del-Recon}. 
  
  We start by giving formal definitions. We define a text sequence as $\mathbf{x} = (x_1, x_2, \dots, x_t)  \in \cX^t$,  where each $x_i \in \cD$ is a word,  $\cD$  is a set of words that  shapes a predefined dictionary. 
%
%
A (next-step) language model is a generative model which models the probability $\Pr[\mathbf{x}]$ by applying the chain rule $\Pr[\mathbf{x}] = \Pi_{1}^t \Pr[x_i \mid x_1, x_2, \dots ,x_{i-1}] $. Specifically, a next-step language model $f$ takes a prefix of the text sequence $x_1,\dots,x_i$ as input, and with the parameter $\theta$ it returns the likelihood $f_\theta(x_i \mid x_1, \dots x_{i-1})$ that ideally equals $\Pr[x_i \mid x_1, x_2, \dots ,x_{i-1}]$. 
As an example, an $N$-gram language model models the mentioned probabilities with a Markov chain, and it approximates $\Pr[x_i \mid x_1, \dots ,x_{i-1}]$ with the estimated probability of $N-1$ previous words, i.e., $\Pr[x_i \mid x_{i - N +1}, \dots ,x_{i-1}]$ {($N$ contiguous words $x_{i-N+1}, \dots, x_i$ is called an $N$-gram)}. Specifically a bigram language model ($N=2$) follows $f_\textbf{bi}[x_i \mid x_1, \dots ,x_{i-1}] \approx \Pr[x_i \mid x_{i-1}]$ and a trigram language model ($N=3$) follows $f_\textbf{tri}[x_i \mid x_1, \dots ,x_{i-1}] \approx \Pr[x_i \mid x_{i-2}, x_{i-1}]$.
In the training of the language model, a training dataset $\cS$ with multiple sequences is given. {The language model parameter $\theta$ is optimized to maximize the overall likelihood of sequences in the training dataset, that is, the probability of returning the dataset given such $N$-gram probability.}

\paragraph{Threat model.} 
In general, we follow the security game described in Definition \ref{def:Del-Recon}. Namely, we first train the language model with a dataset $\cS$. In the deletion step, we delete a random sequence $\mathbf{x'} = x'_1, \dots ,x'_{i-1}$ from the dataset and retrain the model. Finally, the adversary aims to reconstruct the example $\mathbf{x'}$. Note that black-box access by the adversary means that it can send an text sequence $x_1, x_2, \dots, x_t$ to the language model and gets the probability of the text sequence $f_{\theta}(x_1, x_2, \dots, x_t)$.

\paragraph{Our deleted sentence reconstruction attack.} We now define a simple adversary that can accurately reconstruct the deleted sentence. It first simply queries every possible $N$-grams in the dictionary to the model $\model$ and records their probabilities. Then after deletion, it again sends every possible $N$-grams queries to the model $\modelDel$. 
Now suppose $N=2$, i.e. bigram. According to the definition of language models, for a word pair $(x'_{i-1}, x'_i)$, if $\model(x'_{i-1}, x'_i) > \modelDel(x'_{i-1}, x'_i)$, then the number of occurrence of the bigram $x'_{i-1}, x'_i$ is decreased in the updated dataset, which further indicates the bigram is included in the deleted example. {Therefore, for one particular suffix $x'_{i-1}$, the adversary can guess a word $x'_i$ which satisfies that $\model(x'_{i-1}, x'_i) > \modelDel(x'_{i-1}, x'_i)$.}

We then propose a heuristic approach to reconstruct the deleted text sequence. First, we abstract the problem into a search problem defined on a graph, where each node is a $N$-gram that satisfies $$\model(x'_{i-N+1}, \dots, x'_i) > \modelDel(x'_{i-N +1}, \dots, x'_i).$$ We draw a directed edge from an $N$-gram node $v_i$ to $v_j$ if and only if the last $N-1$ words of $v_i$ is the first $N-1$ words of $v_j$ with the same order. Then each path in the graph represents a sentence.
We then search to find a Hamiltonian path in the generated graph. 
Note that it is possible that the deleted sentence includes a specific $N$-gram with multiplicity more than one. We then allow the ``Hamiltonian'' path  to tolerate a limited number of repetitive visits to a node. Finally, we return the shortest traverse path found, i.e., with the fewest number of repetitions. To implement this attack we use a recursive algorithm to traverse the nodes of the graph while we maintain the number of times that the current path has visited each node.

\paragraph{Experiment details.} We  perform our attacks on unigram, bigram, and trigram language models. We train the language models on the Penn Treebank Corpus~\cite{marcus1993building}. After regular preprocessing, the dataset includes 42068 text sequences, which includes 971657   words and 10001 unique words.
We use two metrics to evaluate our attacks.

\begin{itemize}
    \item Success rate: Probability that the adversary reconstructs a sequence $x$ completely, when $x$ is chosen at random from $\cS$, it is deleted, and then the adversary is able to extract $x$ by first interacting with $\model$ and then with $\modelD{x}$.
\item F1 score of the reconstruction: Let the reconstructed sequence of the adversary $ \Adv$  be $x'$ and the deleted sequence be $x$. Let's treat both of them as unordered multisets. Then the F1 score of the reconstruction measures the quality of the reconstruction by balancing the precision and recall of the prediction, namely, 
$$F1 = \frac{2}{\frac{1}{\mathrm{Precision}} + \frac{1}{\mathrm{Recall}}} = \frac{2|x \cap x'|}{ (|x|+|x'|)}.$$ 
which is equal to $1$ if and only if $x=x'$ (as multisets).
\end{itemize}

We then repeat the security game for 1000 times (i.e., each time a random sentence is deleted), and measure the two metrics on three language models.

\paragraph{Results.} We present the experimental result on the three language models in Table \ref{tab:language_model}. Note that the unigram language model does not store anything on order, so it is impossible to reconstruct the full sequence in the correct order. Our defined reconstruction attack gets 99\% on the bigram and trigram models on the F1 score and   successfully reconstruct 97\% of the sequence with correct words and correct order on the trigram model.

\begin{table}[ht]
\centering
\begin{tabular}{l|l|l}
\hline
        & Success rate & F1 score\\ \hline
        
unigram & \textbackslash &  93.76\% \\ \hline
bigram  & 62.00\%                                             & 99.72\%                                    \\ \hline
trigram  & 97.30\%                                             & 99.90\%                                    \\ \hline
\end{tabular}%
\caption{Reconstruction attack on language models.}
\label{tab:language_model}
\end{table}


\paragraph{Leakage of deletion.} Note that without deletion, even if the adversary can fully reconstruct the $N$-gram model, the adversary only has the probability of $N$-grams, which is an aggregation over all the $N$-grams in the dataset. Although the adversary has those $N$-grams, it is still hard to get specific private information when the dataset is large. However, we show that when deletion happens, by tracing the \emph{changes} in the probabilities during the deletion, an adversary can extract the \emph{full deleted sequence} (of length longer than $N$) with high probability, completely revealing the deleted sequence.

\subsection{Deleted Label Reconstruction} \label{sec:delLabelRec}
We now show that when the dataset is small, the label $y$ for the deleted example $e=(x,y)$ might completely leak through a black-box access to the models before and after the deletion. Note that when $y$ is binary, there is little difference between label \emph{inference} and \emph{reconstruction}, but our attacks work even when the labels are not binary, hence it is suitable to call them label reconstruction attacks as defined in Definition \ref{def:Del-Recon}.

We propose the following attack $\LabelInf$ to reconstruct the deleted label.

\begin{alg}[Attack $\LabelInf$]
Given models $\model$ and $\modelDel$, a number $n \in \mathbb{N}$, the label inference attacker $\LabelInf$  proceeds as follows:
\begin{enumerate}
    \item Randomly pick $m$ random samples $\cT = \{x_1, x_2, \dots, x_m\}$ in the data range $\cX$.
    \item For all $i \in [m]$, query   $\model$  to obtain $\hat{y}_i = \model(x_i)$.
    \item For all $i \in [m]$, query   $\modelDel$ to  obtain $\hat{y}'_i = \modelDel(x_i)$.
    \item Return 
    $\argmin_{c \in \cY} 
    \sum_{i=1}^m (\Pr[\hat{y}'_i = c] - \Pr[\hat{y}_i = c]).$
\end{enumerate}
\label{alg:labelinf}
\end{alg}

The intuition   is that for natural models (e.g., ERM rule), removing one example of a specific class will tend to move the prediction towards other classes, i.e., the expectation of predictions to that specific label is likely to decrease. In other words, if the attack above fails, it means that adding this deleted sample back to the training set will let the model tend to predict other classes, which is an unlikely scenario. Our experiments confirm that this attack intuition succeeds.

\paragraph{Experiment details.}
We test the attack on three  classification datasets, including the Iris Dataset \cite{fisher1936use}, the Wine Recognition dataset \cite{aeberhard1994comparative}, and the Breast Cancer Wisconsin Diagnosis dataset \cite{street1993nuclear}. The label is among a discrete set $\cY$. The learning algorithms are the logistic regression model and $K$-Nearest Neighbor model.
 The experiment result is presented in Table~\ref{tab:res-new}. The success probability of the attack is higher than $90\%$ on the Iris and Wine datasets, and is higher than $75\%$ on Breast Cancer dataset.
\begin{table}[ht]
\begin{tabular}{llll}
\hline
                    & Iris    & Wine    & Breast Cancer \\ \hline
Logistic Regression & 92.90\% & 97.30\% & 86.60\%       \\ \hline
K Nearest Neighbor  & 93.70\% & 90.10\% & 77.80\%      \\ \hline
\end{tabular}
\centering
\caption{Results of our deleted label reconstruction attacks}
\label{tab:res-new}
\end{table}

\subsubsection{Known-Instance Label Reconstruction} \label{sec:approx}
We now we study   attacks in which the adversary knows the   instance $x$ of the deleted record $e=(x,y)$ and wishes to  approximate the true label $y$  by querying the models $\model$ and $\modelDel$. The goal is to beat the correctness of both models for true label $y$. 
This means that, in case the two models were supposed to hide the label (perhaps if it was a sensitive information to know very precisely) the data removal process, in this case, clearly goes against the goal of hiding $y$ in its exact form.  


\begin{definition}[Known-instance  label reconstruction] \label{def:Del-Label-App-Recon}
This definition is identical to Definition \ref{def:Del-Recon} with  the only difference that  the adversary is now given $x_i$ (but not $y_i$) in Step \ref{Step:Chal} of the attack. 
\end{definition}
Even though one can define the success criteria of the attackers of Definition \ref{def:Del-Label-App-Recon}   the same way as those of Definition \ref{def:Del-Recon}, such attacks are only interesting if they can beat the precision of the answers provided by the two models $\model,\modelDel$, as anyone (including the adversary) could query those models on the point $x_i$, once $x_i$ is revealed. Our experiments show that such ``accuracy boosting'' attacks are indeed sometimes possible in the presence of deletion operations.

We propose a simple attack $\LabelApp$ in Construction \ref{alg:labelapp} below.    
$\LabelApp$ makes an estimation on $y$ based on the output of the two models.

\begin{alg}[Attacker $\LabelApp$]
This attack is parameterized by $\lambda > 0$.
Given sample $x$, models $\model$ and $\modelDel$, and a constant $\lambda$, the label reconstruction adversary $\LabelApp$ proceeds as follows:
\begin{enumerate}
    \item Query   to obtain $\hat{y} = \model(x)$ and $\hat{y}' = \modelDel(x)$.
    \item Return $\tilde{y} =  \hat{y} + \lambda \cdot (\hat{y} - \hat{y}')$. \qedhere
\end{enumerate}
\label{alg:labelapp}
\end{alg}
\paragraph{Intuition behind the attack.} Similar to the attacks of Section  \ref{sec:del-infer} (see Proposition \ref{prop:changes}), the loss of the deleted sample will increase after the deletion. 
For simplicity, suppose the loss is mean squared error. In this case, when the learner follows the ERM  rule, we have $|\hat{y}' - y|_2 \geq |\hat{y} - y|_2$.
Therefore, moving from $\hat{y}'$ towards $\hat{y}$ makes the prediction closer to the actual label $y$. Consequently, using a small positive $\lambda$ could lead to less loss. The best value of $\lambda$ in each different scenario could be empirically estimated by a similar size dataset that is individually sampled by the attacker.

\paragraph{Experiment details.} We perform the attack on  linear regression models. We test the attack on two classic regression datasets, the Boston Housing Price Dataset \cite{harrison1978hedonic} and the diabetes dataset \cite{efron2004least}. For each dataset, we train the model $\model$ with the whole dataset. 
The adversary returns an approximation $\tilde{y}$.  $|\tilde{y} - y|_2$ will denote the distance of the prediction by the adversary, and we use $\min(|\model(x) - y|_2, |\modelDel(x) - y|_2)$ as the baseline value to compare the quality of adversary's prediction.

\paragraph{Results.} We calculate the average distance of $\Tilde{y}_i$ and $y_i$ with different $\lambda$ values.  Our results (in Table \ref{tab:res-lr}) show that there exists a $\lambda$ value for each dataset, such that can reduce the the estimated loss by around 70\%.



\begin{table}[ht]
\centering
\begin{tabular}{llllll}
\hline
{ } &
  { Best $\lambda$} &
  {   Models} &
  {   Adversary} &
  { $\%$}\\ \hline
{ Boston } &
  { 17.5} &
  { 21.897} &
  { 7.149} &
  { 30\%} \\ \hline
{ Diabetes} &
  { 30} &
  { 2859.7} &
  { 829.8} &
  { 28\%} \\ \hline
\end{tabular}\vspace{0.2cm}
\caption{Result of the label reconstruction Attack on Logistic Regression. The column Models lists the  average of the minimum distance of the predictions of the two models $\model,\modelDel$. The column Adversary lists   the average distance of the prediction of the adversary and the real prediction, and the percentage shows the percentage of the improvement in the prediction  compared with the better of the predictions of the two models $\model,\modelDel$.}
\label{tab:res-lr}
\end{table}

\section{Weak Deletion Compliance} 
\label{sec:GGV}

In Sections~\ref{sec:del-infer} and \ref{sec:del-Rec}, we studied \emph{attacks} on data privacy under data deletion. The definitions of those sections  provide \emph{weak} guarantees on what adversary cannot do, hence they are suitable for stronger \emph{negative} results.
In this section, we investigate the other side; namely, positive results that can prevent attacks of Sections~\ref{sec:del-infer} and \ref{sec:del-Rec} and  provide strong guarantees about what adversary can(not) learn about the data that is being updated through deletion requests.
In particular, we observe that the deletion compliance definition of Garg, Goldwasser and Vasudevan~\cite{GGV20} would prevent attacks of Sections~\ref{sec:del-infer} and \ref{sec:del-Rec}. More precisely, we show that even a  \emph{weaker} variant of the \cite{GGV20} definition would prevent the attacks of Sections~\ref{sec:del-infer} and \ref{sec:del-Rec}. 

\paragraph{Components of deletion compliance definition.} The ``deletion compliance'' framework of Garg, Goldwasser and Vasudevan~\cite{GGV20} provides an intuitive way of capturing data deletion guarantees in general systems that collect and process data. This framework models the world by three interacting parties -- the data collector $\DataCol$, the deletion-requester $\DelReq$, and the environment $\Env$. All components are the same as those of \cite{GGV20}, however,  we will work with a modified $\DelReq$ and a different indistinguishability guarantee.
\begin{itemize}
        \item {\em Data collector (learner)  $\DataCol$}  represents the algorithm that collects the records (training examples) and processes data according to a (learning) mechanism.  For example $\DelReq$ might accept up to $n$ data storage requests and up to $k$ data deletion requests.
    \item {\em Deletion requester (user) $\DelReq$}   is a special honest user who only stores two particular examples $e_0,e_1$ and will delete one of them later. The timing of such requests are stated below. In the original \cite{GGV20}, the deletion requester just stores \emph{one} record $e$ and delete it, or that it might never store $e$ in the first place. At a high level, their $\DelReq$ is designed so that one can define privacy that even hides the deletion itself, while our variant is designed for a weaker definition that does not aim to hide the fact that some deletion has happened.
    \item {\em Environment (adversary) $\Env$}  models the ``rest of users'' who might not be honest and who are interested in finding out what $\DelReq$ is deleting. The interaction between $\Env$ and $\DataCol,\DelReq$ is defined by  the interfaces of $\DataCol,\DelReq$.
\end{itemize} 

\paragraph{Interaction of the components.} 
We let  $\Records $ model a \emph{universe of records}. For example, $\Records = \Supp(D)$ for a distribution over labeled examples $\cX \times \cY$.
We now describe the restrictions on how the components interact with each other. Other than the below-mentioned restrictions, the parties run in PPT.
\begin{itemize}
    \item  $\DataCol$ accepts instructions   $\Add(e), \Del(e), \Eval(x)$. The interpretation of these instructions are as follows. $\Add(e)$ adds the record $e \in \Records$ to the set of records stored at $\DataCol$. $\Del(e)$ removes $e$ from the set stored by the data collector, and $\Eval(x)$ returns the evaluation of the ``current model'' stored by $\DataCol$ (which is the result of learning over the set  stored at $\DataCol$) on $x$ and returns the answer.
    \item As in \cite{GGV20}, we also require that only $\Env$ can send messages to $\DelReq$. At some point in the execution of the system $\Env$ sends $\DelReq$ the following messages, which is   followed by messages from $\DelReq$ to $\DataCol$ as described below.

    \begin{enumerate}
        \item  $(\Add, e_0,e_1)$:   $\DelReq$   sends $\Add(e_0),\Add(e_1)$ to $\DataCol$.
        \item $\Del$:   $\DelReq$ will send $\Del(e)$ to $\DataCol$ where $e \in \set{e_0,e_1}$. By $\DelReq_b$ we refer to the instantiation of $\DelReq$ that sends $\Del(e_b)$ to $\DataCol$.
    \end{enumerate}
\end{itemize}

\paragraph{Weak deletion compliance.} For our purposes, we consider a  different weaker definition (compared to that of \cite{GGV20}) that still captures all attacks of Section~\ref{sec:del-infer} and \ref{sec:del-Rec}. To start, we define two worlds, World 0 and World 1, corresponding to the instantiation of $\DelReq$ by $\DelReq_0$ and $\DelReq_1$.

\begin{definition}[Weak deletion compliance] \label{def:DelComp} Let the interactive algorithms $\DataCol,\Env,\DelReq$ be, in order, the data collector, the environment, and the deletion requester (interactive) algorithms limited to interact as described above.  We call $\DataCol$  \emph{$\eps$   deletion compliant}, if no PPT $\Env$ can detect whether it is in World 0 (with $\DelReq_0$) or World 1 (with $\DelReq_1$) with advantage more than $\eps$.  If this holds under the restriction that $\Env$ makes at most $(k-1)$ deletion requests during the execution, then $\DataCol$ is said to be $\eps$-weak deletion-compliant for up to $k$ deletions \end{definition}

\remove{
We place the following restrictions on the operation of the deletion requester $\DelReq$:
\begin{itemize}[topsep=5pt]
    \item After it starts execution, $\DelReq$ sends \emph{two} records $e_0$ and $e_1$ to the data collector $\DataCol$.
    \item When it is instructed to delete by the environment $\Env$, in World 0 it sends a deletion request to $\DataCol$ for $e_0$, whereas in World 1 it sends a deletion request for $e_1$.
\end{itemize}

The above is also represented in \cref{fig:symm-del-comp}. Given $\Env$ and $\DelReq$ from context, let $\view_E^{0}$ (and $\view_E^{1}$ denote the view of $\Env$ in World 0 (World 1) execution involving $\DataCol$, $\Env$, and $\DelReq$. Finally, since we will only be looking at the view of $\Env$ in our definition, we make one more change to the framework by not terminating the execution with the deletion request from $\DelReq$, but instead allow $\Env$ to call for termination at any point it wishes after this deletion.
}

\paragraph{Comparison with  \cite{GGV20}.}  The key differences between our Definition~\ref{def:DelComp} and that of \cite{GGV20} are as follows. In each case, we state  the property of our definition in contrast to that of \cite{GGV20}.
\begin{itemize}
    \item {\bf Hiding the state of $\DataCol$ from   adversary.}
    The definition of \cite{GGV20} focuses on scenarios where the data collector's state might be revealed at some point in the future (e.g., due to a subpoena). However, in this work we focus on hiding the information that is \emph{leaked} from the data collector (about  deleted record) through   interaction with the adversary.
    \item {\bf Not aiming to hide the deletion itself.} Whereas plain deletion-compliance asks that deletion make the world look as though the deleted data were never present in the first place, here we only ask that it not be revealed \emph{which} record was deleted. For instance, a data collector that is weak deletion-compliant might still reveal the number of deletions it has processed, as long as the data that is deleted is not revealed. While weaker than deletion-compliance definition of \cite{GGV20}, our notion is fit for hiding the deleted record among the records in the training set, and still giving a more general and stronger definition than Definition~\ref{def:WeakDelPriv}.
\end{itemize}


We now formally discuss why   Definition~\ref{def:DelComp} captures the attacks of Section~\ref{sec:del-infer} and \ref{sec:del-Rec}. Recall that Definition~\ref{def:WeakDelPriv} was already shown in Theorem~\ref{thm:fromRecToInf} to be a stronger notion than instance and label reconstruction attacks (Definition~\ref{def:Del-Recon}). Hence, we just need to show   that Definition~\ref{def:DelComp} is   stronger   than Definition~\ref{def:WeakDelPriv}.

\begin{figure*}[ht]
\centering
\includegraphics[width=0.85\textwidth]{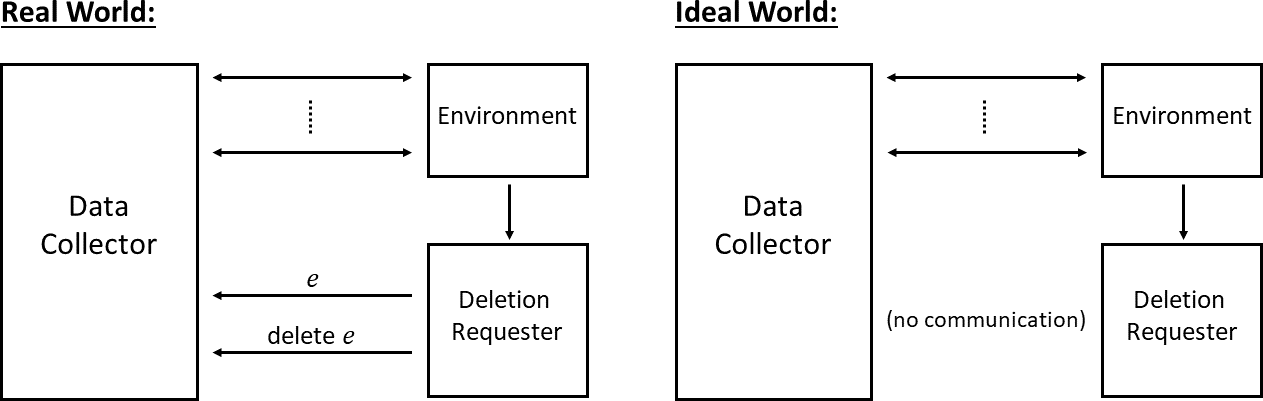}
\captionsetup{justification=centering}
\caption{The real and ideal worlds for (strong) deletion compliance}
\label{fig:del-comp}

\includegraphics[width=0.85\textwidth]{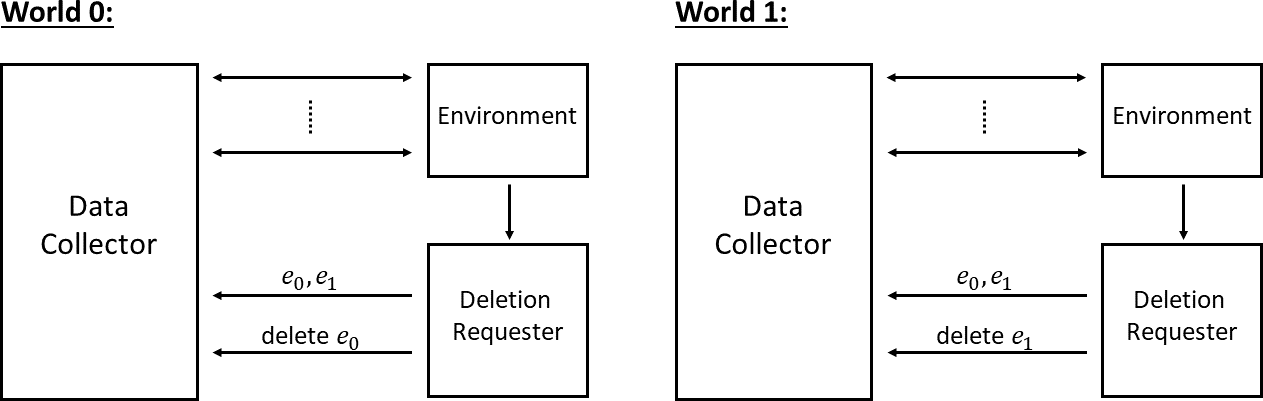}
\captionsetup{justification=centering}
\caption{The worlds for weak deletion-compliance}
\label{fig:symm-del-comp}
\end{figure*}

\begin{theorem}[Deletion inference from compliance] \label{thm:InfComp} Let $\Learn$ be a learner, $\Del$ be a deletion   mechanism for $\Learn$,     $D$ be a distribution over  labeled examples,  and 
$\Records=\Supp(D)$ be the universe of records. The data collector  $\DataCol$  answers queries as follows. 
\begin{enumerate}
    \item $\DataCol$ does not respond any $\Del$ or $\Eval$ queries till receiving $n$ $\Add(\cdot)$ queries; we refer to those $n$ added records as set $\cS$.
    \item $\DataCol$   permutes $\cS$  and gets $h \gets \Learn(\cS)$.
    \item Then it answers $\Eval(e) = h(e)$ queries arbitrarily.
    \item Then it accepts \emph{one} $\Del(e)$, and lets $h_{-e} = \Del(h,e)$.
    \item Then it continues answering $\Eval(e) = h(e)$ queries.
\end{enumerate} 
If $\DataCol$ is $(2\eps-1)$-deletion compliant (as in Definition~\ref{def:DelComp}) against PPT adversaries with oracle access to $D$, then the scheme $(\Learn,\Del)$ is $\eps$-secure against deletion inference (as in Definition~\ref{def:WeakDelPriv}).
\end{theorem}


\begin{proof}[Proof of Theorem \ref{thm:InfComp}]
We give a proof by reduction.
Suppose $\Adv$ breaks the membership inference security game of Definition~\ref{def:WeakDelPriv} with probability $(1+\eps)/2$. We construct an environment $\Env$ that $\eps$-distinguishes $\DelReq_0$ from $\DelReq_0$ with advantage $\eps$ that proceeds as follows:
\begin{enumerate}
    \item $\Env$ plays the role of the challenger from Definition~\ref{def:WeakDelPriv} and picks a data set $\set{e_1,\dots,e_n}=\cS \gets S_n$ of size $n$. $\Env$ passes this   to the $\DataCol$ and  picks $i\neq j \in [n]$ at random as the challenge records.
    \item Next, $\Env$ instantiates $\Adv$ and provides it with the records $e_i, e_j$ and \emph{oracle} access to $\model$ (through $\DataCol$). At the end of this step, the adversary instructs moving to the next step.
    \item $\Env$ passes $(e_i,e_j)$  to $\DelReq$ (which will then request the deletion of one of the two records). 
    \item $\Env$ actives the $\Adv$ again and it is again provided \emph{oracle} access to $\model$ (through $\DataCol$). At the end of this step, the adversary's output is included in the output of the environment. 
\end{enumerate}
The view of the adversary $\Adv$ in the above experiment is identical to its view as part of Definition~\ref{def:WeakDelPriv}. Thus, the output of $\Adv$ will correctly (with probability greater than $\epsilon$) identify whether $\DelReq$ requests the deletion of record $e_i$ or record $e_j$. This allows us to conclude that the view of the $\Env$ changes depending of whether $\DelReq$ requests deletion of $e_i$ or $e_j$.
\end{proof}

Using the same three components described in Section~\ref{sec:GGV} (with a different $\DelReq$), \cite{GGV20} defines the notion of \emph{deletion-compliance}. Here the ideal world is the same as the real world in all respects except that $\DelReq$ is not allowed to communicate with $\DataCol$ as represented in \cref{fig:del-comp}. (The restriction of $\DelReq$ not being able to send messages to $\Env$ was imposed in order for this ideal world to be well-defined, by excluding cases where $\Env$ sends to $\DataCol$ messages that depend non-trivially on $\DelReq$'s records.)
\cite{GGV20}  calls $\DataCol$  to be $\eps$-deletion-compliant if, for any $\Env$ and $\DelReq$, the joint distributions of the state of $\DataCol$ and view of $\Env$ in the real and ideal world are $\eps$-close in the statistical distance, denoted by notation $\approx_{\eps}$. That is, $$(\state_D^{\real},\view_E^{\real}) \approx_{\eps} (\state_D^{\ideal},\view_E^{\ideal}).$$
The above (strong) definition from \cite{GGV20} captures the intuition that a system is deletion-compliant if the state of the world after its deleting a record is similar to what it would have been if the record had never been part of the system in the first place. Note that this requirement of $\eps$-closeness in statistical distance is more relaxed than the kind of closeness of distributions required by differential privacy, and so DP can be used to satisfy these requirements. \cite{GGV20}  showed how to obtain their strong deletion compliance based on differentially private mechanisms.

To contrast with Figure \ref{fig:del-comp}, in Figure \ref{fig:symm-del-comp} we have depicted the more symmetric worlds that are behind our Definition \ref{def:DelComp}. In particular, Definition \ref{def:DelComp} requires that  no PPT $\Env$ can distinguish between World 0 and World 1 of Figure~\ref{fig:symm-del-comp} by more than advantage $\eps$.

{

\bibliographystyle{alpha}
\bibliography{Biblio/abbrev0,Biblio/crypto,Biblio/OtherRefs, Biblio/mlcrypto}
}
{
\appendix

\section{Hyperparameters of Models}
\label{sec:hyperparameter}

Here we describe the hyperparameters of  the models used in the experiments of our paper.
\begin{itemize}
    \item \textbf{MLP:} We use multiple layer perceptron with two hidden layers. For regression, we set the size of hidden layers as $(20, 2)$, and for classification, we set the size of hidden layers as $(20, 10)$. The reason behind is that the output layer of classification tasks have more neurons. We use LBFGS as the optimization algorithm to train the model, and we train $200$ epochs on each model.
    \item \textbf{SVM:} We use the default SVMClassifier and SVMRegressor in Scikit-learn. Specifically, we use the RBF kernel with $C=1.0$.
    \item \textbf{Decision tree:} For the decision tree model, we use the default DecisionTreeClassifier and DecisionTreeRegressor in Scikit-learn. Specifically, we use Gini impurity to split the leafs and do not set a limit on the tree size. 
    \item \textbf{Random forest:} We use the default RandomForestClassifier in Scikit-learn, which generates 10 trees in the forest. For each tree, its hyperparameter is the same with the decision tree classifier above.
    \item \textbf{Logistic Regression and Linear regression:} We use the default LinearRegression and LogsticRegression from Scikit-learn.
    \item \textbf{Lasso regression:} We set $\alpha=0.1$ in   lasso regressor.
\end{itemize}

}
\end{document}